\documentclass[journal]{IEEEtran}

\ifCLASSINFOpdf
\usepackage[pdftex]{graphicx}
\graphicspath{{../pdf/}{../jpeg/}}
\DeclareGraphicsExtensions{.pdf,.jpeg,.png}
\else
\usepackage[dvips]{graphicx}
\graphicspath{{../eps/}}
\fi
\usepackage{graphicx}
\usepackage{epstopdf}
\usepackage{bm}
\usepackage{diagbox}
\usepackage{multirow}
\usepackage[cmex10]{amsmath}
\usepackage[ruled]{algorithm2e} 
\usepackage{url}
\usepackage{lineno,hyperref}
\usepackage{amsthm}
\usepackage{booktabs}
\usepackage{subfigure}
\usepackage{booktabs}
\usepackage{url}
\usepackage{mathrsfs}
\usepackage{amsmath}
\usepackage{amsthm}
\usepackage{bbding}
\usepackage{amsfonts,amssymb}
\usepackage{mathrsfs}
\usepackage{amsmath}
\usepackage{amsthm}
\usepackage{bbding}
\usepackage{cite}
\usepackage{latexsym,amsmath,xcolor,multicol,booktabs,calligra}

\begin{document}
	\title{Robust Brain Tumor Segmentation with Incomplete MRI Modalities Using Hölder Divergence and Mutual Information-Enhanced Knowledge Transfer}
	\author{Runze Cheng$^{\dagger}$,
		Xihang Qiu$^{\dagger}$, 
		Ming Li$^{\dagger}$, Ye Zhang, Chun Li$^{*}$, ~\IEEEmembership{Member,~IEEE}, and F. Yu,~\IEEEmembership{Fellow,~IEEE}
		\thanks{This research was supported by the National Key Research and Development Program of China (No. 2022YFC3310300), Guangdong Basic and Applied Basic Research Foundation (No. 2024A1515011774), the National Natural Science Foundation of China (No. 12171036), Shenzhen Sci-Tech Fund (Grant No. RCJC20231211090030059), and Beijing Natural Science Foundation (No. Z210001). The code is available at the following \url{https://github.com/Lboy11/Brain-Tumor-Segmentation.git}.}
		\thanks{R. Cheng, X. Qiu, Y. Zhang, and C. Li are with MSU-BIT-SMBU Joint Research Center of Applied Mathematics, Shenzhen MSU-BIT University, Shenzhen, 518172, China. Email: 3120220916@bit.edu.cn; qiuxh@bit.edu.cn; ye.zhang@smbu.edu.cn.} 
        \thanks{R. Cheng and X. Qiu are also with Institute of Control Theory and Control Engineering, School of Automation, Beijing Institute of Technology, 100081, Beijing, China.}
        \thanks{Y. Zhang is with School of Mathematics and Statistics, Beijing Institute of Technology, 100081, Beijing, China.}
        \thanks{M. Li and F. Yu are with Guangdong Laboratory of Artificial Intelligence and Digital Economy (SZ), Shenzhen, 518083, China. E-mail: ming.li@u.nus.edu; yufei@gml.ac.cn.}
		\thanks{$^{\dagger}$Co-first authors. *Corresponding author: Chun Li (E-mail: lichun2020@smbu.edu.cn).}
	}
	\markboth{IEEE/CAA Journal of Automatica Sinica}%
	{Shell \MakeLowercase{\textit{et al.}}: Bare Demo of IEEEtran.cls for IEEE Journals}
	\maketitle

	\begin{abstract}
                Multimodal MRI provides critical complementary information for accurate brain tumor segmentation. However, conventional methods struggle when certain modalities are missing due to issues like image quality, protocol inconsistencies, patient allergies, or financial constraints. To address this, we propose a robust single-modality parallel processing framework that achieves high segmentation accuracy even with incomplete modalities. Leveraging Hölder divergence and mutual information, our model maintains modality-specific features while dynamically adjusting network parameters based on available inputs. By using these divergence and information-based loss functions, the framework effectively quantifies discrepancies between predictions and ground-truth labels, resulting in consistently accurate segmentation. Extensive evaluations on the BraTS 2018 and BraTS 2020 datasets demonstrate superior performance over existing methods in handling missing modalities, with ablation studies validating each component’s contribution to the framework.
	\end{abstract}
	
	
	\IEEEpeerreviewmaketitle

\section{Introduction}
Brain tumors are aggressive diseases requiring early, accurate detection to improve treatment outcomes. Magnetic Resonance Imaging (MRI) is widely regarded as an essential tool for assessing brain tumors, thanks to its high-resolution soft tissue contrast and non-invasive nature \cite{pp27}. MRI-based brain tumor segmentation is critical in delineating healthy tissue from tumorous regions, which supports diagnosis, tracking tumor progression, and assessing treatment response. Typically, multimodal MRI—consisting of T1-weighted, T1-weighted post-contrast (T1c), T2-weighted, and FLAIR images—provides complementary insights, with each modality capturing specific tissue characteristics necessary for precise tumor boundary identification \cite{pp28,pp29,pp32}.

Most traditional segmentation methods assume access to all four MRI modalities. However, in clinical practice, it is common for one or more modalities to be unavailable, resulting in incomplete multimodal data. Missing modalities can occur due to various reasons, such as differences in imaging protocols, patient allergies to contrast agents, image quality issues from patient movement, or resource limitations in healthcare facilities \cite{p1}. These scenarios of missing modalities pose significant challenges for existing segmentation methods, which often experience performance drops without access to the full imaging data.

To address the missing-modality issue, two main approaches have emerged: modality-specific methods and unified single-model methods. In modality-specific approaches, separate models are trained for each missing modality scenario, allowing for knowledge transfer from well-trained multimodal networks to those designed for limited modalities. For instance, SMU-Net \cite{p5} utilized a teacher-student framework, where a multimodal teacher network transfers knowledge to unimodal student networks. ProtoKD \cite{p6} combined knowledge distillation with prototype learning to capture essential data patterns, while MMCFormer \cite{p7} employed transformers with auxiliary tokens for modality-specific representation transfer. Although these methods can achieve high segmentation accuracy, they are computationally intensive, especially with more modalities, as each missing-modality case requires a unique model, leading to exponential growth in the number of models.

The second approach addresses all missing-modality scenarios with a unified model, typically by projecting each modality through independent encoders into a shared latent space for feature fusion. For instance, RFNet \cite{p44} adopts a region-aware fusion strategy; Ting and Liu \cite{p47} introduce modality-specific encoders coupled with a shared decoder; and Wang et al. \cite{p9} combine shared and modality-specific features to enhance segmentation performance. Recent methods, including GGMD \cite{p48}, MAVP \cite{pp11}, and QuMo \cite{pp12}, further aim to improve robustness under missing modalities. Specifically, GGMD applies gradient-guided modality decoupling to reduce overreliance on dominant modalities but neglects optimization for asymmetric data distributions and underutilizes cross-modal complementarity. MAVP employs a modality-status classifier to generate visual prompts that guide feature fusion, yet it is susceptible to classifier inaccuracies and does not address data distribution asymmetry. QuMo adapts to missing modalities via learnable queries but fails to adequately mitigate dominant modality effects and lacks explicit metrics for handling asymmetry. In summary, although unified models effectively reduce computational overhead, they typically struggle to maintain robust and consistent performance under complex missing-modality conditions or imbalanced data distributions.

To overcome these limitations, we propose a single-modality parallel processing framework \cite{pp34,pp35,pp36,pp37} that integrates Hölder divergence and mutual information for enhanced robustness in missing-modality scenarios. Inspired by Chang et al. \cite{p36} on mutual information-based learning and recent advancements in high mutual information knowledge transfer \cite{p52}, our approach processes each modality independently through a shared network backbone, preserving modality-specific information and allowing flexible handling of missing modalities. Specifically, we introduce a mutual information-based metric with Hölder divergence \cite{p37} to robustly measure discrepancies between predicted and true segmentation maps, promoting effective feature alignment across modalities.

The main contributions of this work include: \textbf{1. Novel Network Architecture: A parallel 3D U-Net-based structure that processes each modality independently and dynamically adjusts network parameters for missing modalities, achieving high segmentation accuracy even with incomplete data. } \textbf{2. Enhanced Robustness with Hölder Divergence and Mutual Information: By incorporating Hölder divergence and mutual information into the loss function, our method effectively measures prediction-label discrepancies, improving robustness against data distribution asymmetries.} \textbf{3. Extensive Validation and Ablation Studies: Comprehensive evaluations on the BraTS 2018 and BraTS 2020 datasets demonstrate that our framework consistently outperforms existing methods across various missing-modality scenarios. Ablation studies further validate each component’s significance in achieving robust segmentation.}

In conclusion, this study introduces a robust solution to the missing-modality problem in brain tumor segmentation, addressing both computational efficiency and segmentation accuracy challenges. By harnessing the benefits of Hölder divergence and mutual information for effective knowledge transfer, our framework offers a resilient approach for multimodal segmentation in diverse clinical environments. This work is an extended version of the conference paper \textbf{Robust Divergence Learning for Missing-Modality Segmentation \cite{pp19}}, which has been accepted for presentation at \textbf{the China Automation Congress 2024}. Compared to the conference version, this work extends the original work in three primary aspects: (1). We redesign and present a more rigorous and systematic network architecture diagram to improve the clarity and precision of the proposed methodology; (2). We conduct more extensive experimental analyses, comprehensively benchmarking our approach against state-of-the-art methods addressing the missing modality segmentation problem, thus verifying its superiority and state-of-the-art performance; and (3), We provide an in-depth theoretical analysis, including Theorems 1 through 7, which rigorously demonstrates the validity and robustness of the proposed method, significantly enhancing the methodological rigor and theoretical foundation of this manuscript.

\section{Related Works}

\subsection{Incomplete Multi-Modal Brain Tumor Segmentation} 
Deep learning has been widely applied in various scenarios such as intelligent transportation \cite{li2021self,li2021exploiting,Yan_2025_CVPR}, cross-modal learning \cite{Liu_2025_CVPR, zhao2025favchat, li2025uni}, privacy protection \cite{li2023dr,li2023stprivacy}, and artificial intelligence generated content \cite{li2024instant3d, liu2024realera,zhuang2025vistorybench}. 
For brain tumor segmentation, compared to methods that assumed all modalities are available \cite{p2}, segmentation approaches handling missing modalities faced greater challenges but offered more practical flexibility for real-world applications. For example, HeMIS \cite{p15} addressed the issue of missing modalities by learning modality embeddings and mapping them into a shared latent space, effectively avoiding traditional data imputation techniques. ACN \cite{p10} and SMU-Net \cite{p5} introduced strategies to infer missing features from complete modality models. ACN utilized a teacher-student framework for each missing modality, while SMU-Net reconstructed missing information from a complete modality network through style-matching mechanisms. mmFormer \cite{p45} enhanced model robustness by employing independent convolutional encoders, an additional transformer module, and auxiliary regularizers tailored for incomplete data scenarios. Similarly, Chen et al. \cite{pp12} proposed a transformer-based end-to-end model that tackled missing multi-modal issues by leveraging a single autoencoder with learnable modality combination queries and a retraining strategy, significantly improving segmentation performance on small datasets with incomplete data. Swin UNETR \cite{pp33} employs a hierarchical Swin Transformer encoder that transforms multimodal input data into serialized feature embeddings. It applies a shifted-window strategy to compute self-attention, effectively capturing features across multiple scales. These features are then connected to a fully convolutional neural network decoder through skip connections, enhancing the model’s capability for multi-scale perception. MMCFormer \cite{p7} introduces an efficient 3D Transformer module, a global context interaction module, auxiliary tokens, and a feature consistency loss to effectively strengthen multi-scale feature representation and maintain feature consistency during training.

Most of these approaches primarily addressed the missing modality problem by reconstructing missing modalities or incorporating complex network structures with self-attention mechanisms. However, they often neglected the need to learn features unique to each modality while concurrently extracting shared features. To resolve this limitation, Wang et al. \cite{p9} proposed a method that simultaneously learned shared and modality-specific features by combining distribution alignment, domain classification, and residual feature fusion. This approach effectively utilized available modalities and managed missing modalities across various tasks. SFusion \cite{pp14} employed self-attention and modality-attention mechanisms to fuse available modalities without imputation, achieving efficient shared representation of multi-modal features. $\text{M}^2\text{FTrans}$ \cite{pp15} introduced trainable fusion tokens, masked self-attention, spatially weighted attention, and channel-level fusion transformers to address missing modality challenges. Knower et al. \cite{pp9} applied meta-learning and adversarial learning strategies to optimize the use of limited samples, effectively enhancing modality-agnostic representations. Qiu et al. \cite{pp11} developed the Modality-Aware Visual Prompt (MAVP) framework, which generated prompts through a modality state classifier to better handle missing modality scenarios. GSS \cite{pp13} proposed a class-aware group self-support learning framework that leveraged modality complementarity and inter-modality collaboration to reduce bias and compensate for information loss effectively.

Despite progress made by existing methods in handling missing modalities, many relied on complex architectures, reconstruction mechanisms, or self-attention modules, resulting in significant computational overhead. Additionally, these methods often struggled to accurately identify and integrate unique information from individual modalities. To address these challenges, we proposed a streamlined 3D-UNet architecture with a single-modality parallel network that processed each modality independently within a shared backbone. This design allowed the model to retain modality-specific features, ensuring accurate segmentation even in the absence of certain modalities.

\subsection{Hölder Divergence Learning}
Hölder divergence \cite{p37} has garnered significant attention in recent years for its flexibility and effectiveness in measuring differences between probability distributions, particularly in tasks involving asymmetric distributions or noise. Hoang et al. \cite{p11} explored the application of Cauchy-Schwarz divergence, a specific form of Hölder divergence, for analyzing Poisson point processes, providing a theoretical foundation for its use in statistical and machine learning tasks. Frank et al. \cite{p12} demonstrated that Hölder divergence enhanced clustering accuracy in K-means clustering, outperforming traditional distance metrics.

In our study, we applied Hölder divergence to 3D medical image segmentation, utilizing its ability to flexibly control the measurement of distributional differences for optimizing model parameters. This approach enabled sensitivity to significant discrepancies while maintaining robustness to noise, ensuring precise segmentation performance even under conditions of missing or incomplete modality data.

\begin{figure*}[hbt!]
	\centering
	\includegraphics[width=0.8\linewidth]{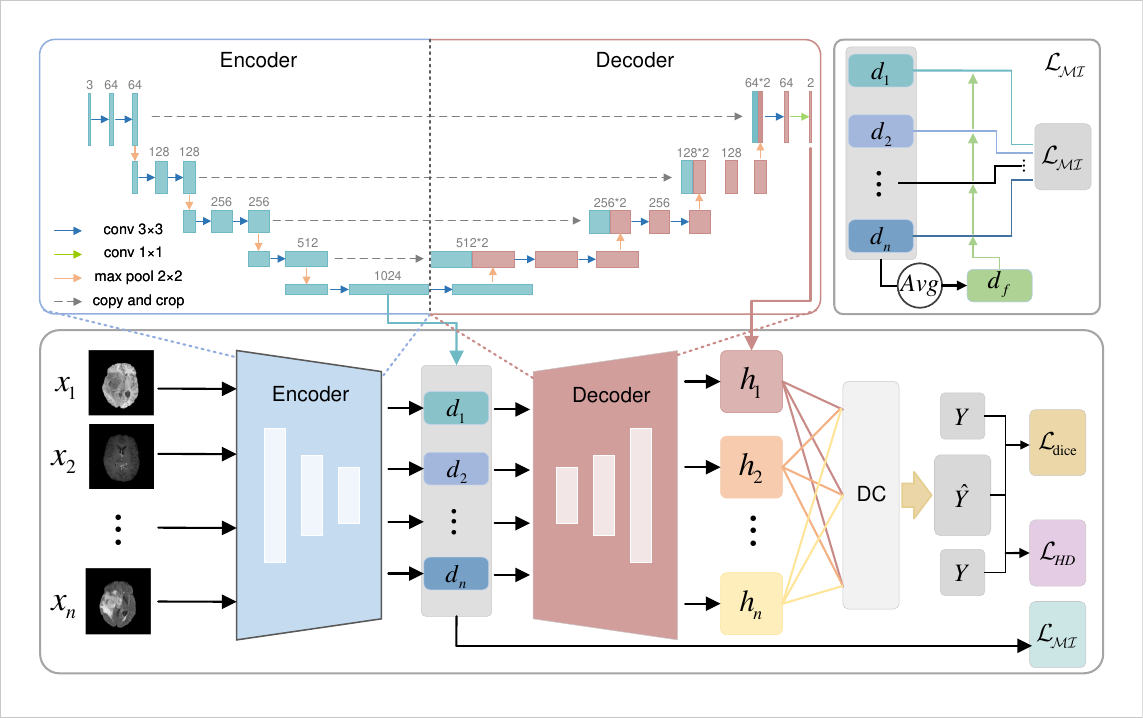}
	\caption{Framework for Robust Brain Tumor Segmentation with Incomplete MRI Modalities Using Hölder Divergence and Mutual Information-Enhanced Knowledge Transfer. This figure presents the structure of the proposed divergence-based learning framework, specifically designed to tackle segmentation challenges in scenarios with missing MRI modalities.}
	\label{fig:_1}
\end{figure*}
\subsection{Knowledge Distillation Techniques}
Knowledge distillation, originally proposed by Hinton et al. \cite{p53}, operated as a teacher-student framework, where a smaller student model learned from a larger teacher model by replicating its “dark knowledge,” which represented inter-class similarity information \cite{p53,pp30,pp31}. In multimodal segmentation tasks involving missing modalities, methods such as ACN \cite{p10}, SMU-Net \cite{p5}, and mmFormer \cite{p7} successfully transferred “deep knowledge” from a full-modal network to a missing-modal network through co-training. However, these co-trained networks were often limited by fixed relationships between a complete modality and a specific missing modality, reducing their adaptability to more diverse applications.

To address these limitations, recent studies introduced innovative knowledge distillation techniques to improve the sharing of semantic information and enhance model robustness. For instance, $\text{M}^3\text{AE}$ \cite{p46} employed a self-distillation mechanism that extracted and shared common semantics across various missing-modality scenarios within a single network, improving segmentation performance while reducing resource consumption. ShaSpec \cite{p9} separated shared features from modality-specific features and combined domain classification with distribution alignment objectives, making the shared features more robust in multimodal environments. GSS \cite{pp13} adopted a self-supported student group as the teacher network, reducing computational costs by eliminating the dependence on large teacher networks.

Our approach addressed the missing modality issue by combining the deepest latent features of each modality within a single-modality parallel architecture based on 3D U-Net. We computed the average of these latent features to create a unified representation for cross-modality knowledge distillation. This strategy effectively managed missing-modality scenarios while enhancing model performance by preserving both modality-specific and shared features. By leveraging this design, our method overcame the limitations of traditional teacher-student frameworks that relied on fixed modality relationships, offering a more flexible and robust solution for missing-modality segmentation.
\begin{table}[t]
	\centering
	\caption{Main Notations Used in This Work. This table provides an overview of the main notations used throughout this work, offering a concise reference for understanding the symbols and terminology employed in the algorithms discussed.}
	\setlength{\tabcolsep}{3pt}
	\begin{tabular}{p{1.8cm}<{\raggedright}|p{6cm}<{\raggedright}} 
		\toprule [1.0pt]
		Notation&Definition\\
		\midrule [0.5pt]
		$x_i$&  The $i^{th}$ modality data of the sample. \\
		$d_i$ & The deepest-level feature of the $i^{th}$ modality. \\
		$d_f$ & The deepest-level full-modality feature. \\
		$h_i$ & The generated single-modality representation of the $i^{th}$ 
modality. \\
		$\widehat{Y}$ & Integrated output under missing modalities. \\
        $Y$ & real sample. \\
        $p(d_f\mid d_m)$ & The conditional distribution of the feature f given the missing modality information $m$.\\
        $q(d_f\mid d_m)$ & The conditional distribution approximated using variational methods.\\

		\bottomrule[1.0pt]
	\end{tabular}
	\label{tab-001}
\end{table}

\section{Methodology}

\subsection{Knowledge Distillation for Segmentation Using Hölder Divergence} 

The main notations used in this work is shown in Table \ref{tab-001}. Brain tumor segmentation, particularly glioma segmentation, involves distinguishing four categories: background, whole tumor, tumor core, and enhancing tumor. Missing modalities can degrade segmentation accuracy. The Hölder divergence is employed for its flexibility and robustness, making it suitable for complex models and non-symmetric data. It supports brain tumor segmentation under missing modalities, maintaining high accuracy in clinical settings.

The loss function using Hölder divergence is: 

\begin{equation}
\frac{1}{D\times H\times W}\sum_{dhw} D_{\alpha}^{H}(\sigma(\mathbf{S}^{p}_{dhw})|\sigma(\mathbf{S}^{l}_{dhw})),
\end{equation}
where $\mathbf{S}^{p}_{dhw}$ and $\mathbf{S}^{l}_{dhw}$ are predicted and label probabilities for pixel $(d, h, w)$, and $D_{\alpha}^{H}$ denotes Hölder divergence, the definition is shown in Definition \ref{def_1}:
\newtheorem{definition}{\bf{Definition}}
\begin{definition} 
	\label{def_1}
	(\textbf{Hölder Statistical Pseudo-Divergence, HPD \cite{p37}}) HPD pertains to the conjugate exponents $\alpha$ and $\beta$, where $\alpha \beta>0$. In the context of two densities, $p(x) \in {L^\alpha }\left( {\Omega,\nu } \right)$ and $q(x) \in {L^\beta }\left( {\Omega ,\nu } \right)$, both of which belong to positive measures absolutely continuous with respect to $\nu$, HPD is defined as the logarithmic ratio gap, as follows: 
\begin{equation}
D_{\alpha}^{H}(p(x):q(x))=-\log\left(\frac{\int_{\Omega}p(x)q(x)\mathrm{d}x}{\left(\int_{\Omega}p(x)^{\alpha}\mathrm{d}x\right)^{\frac1\alpha}\left(\int_{\Omega}q(x)^{\beta}\mathrm{d}x\right)^{\frac1\beta}}\right),
\end{equation}
when $0<\alpha<1$ and $\beta  = \bar \alpha  = \frac{\alpha }{{\alpha  - 1}} < 0$ or $\alpha<0$ and $0<\beta<1.$
\end{definition}

To better highlight its strengths, we emphasize that Hölder divergence provides greater flexibility than traditional divergence measures such as KL divergence or Jensen-Shannon divergence, especially for handling asymmetric distributions. Its mathematical formulation allows for finer control of penalty terms through the conjugate exponents $\alpha$ and $\beta$, enabling adaptation to various data distributions. For example, lower $\alpha$ values improve robustness against outliers, which is critical in clinical scenarios with noisy data.

\begin{algorithm}[t]
	\caption{\small Robust Brain Tumor Segmentation with Incomplete MRI Modalities Using Hölder Divergence and Mutual Information-Enhanced Knowledge Transfer.}
	\label{alg:spl}
	\DontPrintSemicolon
	\small
	\tcp*[f]{\textbf{*Training*}}\\
	\textbf{Input:} Multi-Modality Dataset: $D = \left\{ \left\{ {\rm X}_m^n \right\}_{n = 1}^N, {y_m} \right\}_{m = 1}^M$;\\
	\textbf{Initialization:} Initialize the parameters of the neural network.\\
	\While{not converged}  
	{   
	    (1) Extract the deep features of each single modality and their corresponding segmentation results:\\
         \For{i = 1 to N}
         {
          $d_i \leftarrow f_i(x_i)$ \\
          $h_i \leftarrow T(d_i; \theta)$  
         }
          
        (2) Extract the deep latent features across all modalities: \\
        $d_f \leftarrow \text{Mean}(\{d_1, d_2, \ldots, d_N\})$; \\
       
        (3) Obtain the deep features and corresponding segmen-\\tation results under different missing modality conditio-\\ns:\\
         \For{l in \text{reversed(range(1, M + 1))}}
         {
            \For{subset in \text{itertools.combinations(range(M), l)}}
            {
                $d_m \leftarrow \text{Mean}(\{d_i \,|\, i \in \text{subset}\})$;\\
                $\widehat{Y} \leftarrow \text{Mean}(\{h_i \,|\, i \in \text{subset}\})$;
            }
         }
        
        (4) Compute the total loss:\\
        $\mathcal{L}_{all} \leftarrow \mathcal{L}_{Dice}(\widehat{Y}, Y) + \lambda_1 \mathcal{L}_{\mathcal{MI}}(d_f, d_m) + \lambda_2 \mathcal{L}_{HD}(\widehat{Y}, Y)$; \\
	}
	\textbf{Output:} Segmentation map.\\
	\tcp*[f]{\textbf{*Testing*}}\\
	Evaluate segmentation using the Dice Similarity Coefficient (DSC).\\
\end{algorithm}

\subsection{High Mutual Information Knowledge Transfer Learning}

In clinical practice, missing modality segmentation poses significant challenges, including incomplete information, limited model generalization, and restricted data applicability. To address these issues, we evaluate our approach under real-world missing modality scenarios, providing robust validation beyond controlled datasets. Our method leverages knowledge transfer between full and missing modalities, maximizing the use of available information to compensate for data loss. This enhances model accuracy and stability even with incomplete datasets. The approach extracts feature vector pairs from the encoder layers of both full and missing modality pathways. By calculating entropy and conditional entropy for each pair, we quantify mutual information, measuring how the full modality pathway reduces uncertainty associated with the missing modality. A variational information maximization technique is employed to estimate these mutual information values accurately.

Specifically, we extract $K$ pairs of feature vectors $\left\{ \left( d_f^{(k)}, d_m^{(k)} \right) \right\}_{k=1}^K$ from the encoder layers of both the full and missing modality paths. By calculating the entropy $H(d_f)$ and conditional entropy $H(d_f\mid d_m)$ for each pair, we derive the mutual information 
\begin{equation}
MI(d_f;d_m)=H(d_f)-H(d_f \mid d_m),
\end{equation}
which shows how the full modality path reduces uncertainty given the missing modality information. To estimate these mutual information values accurately, a variational information maximization method \cite{p54} is employed.

We approximate the conditional distribution $p(d_f\mid d_m)$ with the variational distribution $q(d_f\mid d_m) $ to optimize the layer-wise mutual information. The optimization process is defined by the following loss function
\begin{equation}
\mathcal{L}_{\mathcal{MI}}=
-\sum_{k=1}^{K} \gamma_{k} \mathbb{E}_{d_f^{(k)}, d_m^{(k)} \sim p\left( d_f^{(k)}, d_m^{(k)} \right)} 
                           \left[ \log q\left( d_f^{(k)} \mid d_m^{(k)} \right) \right].
\end{equation}

In our framework, the parameter $\gamma_{k} $ increases with the layer level $k$, reflecting the richer semantic information in higher network layers. This ensures effective knowledge transfer by assigning higher weights to these layers. The implementation of the variational distribution $ - \log q(d_f\mid d_m)$ is given by: 
\begin{equation}
\sum_{c=1}^{C} \sum_{h=1}^{H} \sum_{w=1}^{W} \left( \log \sigma_c + \frac{\left( d_f^{c,h,w} - \mu^{c,h,w}(d_m) \right)^2}{2\sigma_c^2} \right) + \text{constant,} 
\end{equation}
where \( \mu(\cdot) \) and \( \sigma \) represent the heteroscedastic mean and homoscedastic variance of the Gaussian distribution, respectively. \( W \) and \( H \) denote the width and height of the image, \( C \) represents the number of channels, and \( \mathrm{constant} \) is a fixed term.

\subsection{Overall Framework}
The proposed framework processes multimodal data by encoding samples into a unified feature space through channel encoders. Each modality is processed independently using a shared backbone network to generate modality-specific representations. A Dynamic Combination (DC) module integrates these representations, adapting to scenarios with missing modalities. For a clearer understanding, Fig. \ref{fig:_1} visually illustrates the framework, showcasing the mechanisms of the DC module and its handling of various scenarios. The optimization process is guided by a Dice loss function, ensuring consistency between the predicted segmentation and target labels. The total loss incorporates contributions from the Dice loss, mutual information loss, and Hölder divergence loss.  The computational complexity arises primarily from mutual information and Hölder divergence calculations. To maintain scalability, GPU acceleration is employed, achieving competitive training times compared to existing methods while improving segmentation performance under missing modality conditions.

Let $X$ and $Y$ denote samples from a multimodal dataset, where $X=\{x_j\}_{j=1}^M$ consists of $M$ samples, each containing $N$ modalities: $x_j = \{x_j^i\}_{i=1}^N$, with $ x_j^i \in X$ representing the $i^{th}$ modality of the $ j^{th}$ sample. The corresponding label is $y_j$. To leverage features from different modalities, a parallel 3D U-Net-based network is designed. Data from each sample is first encoded into a common feature space by channel encoders $f_i$, creating unified representations across channels. The modality-specific data is then processed through a parameterized shared backbone network $T(\cdot; \theta)$, yielding unique single-modality representations $h_i = T(f_i(x_i); \theta).$

The final output is generated using the DC module. For samples with missing modalities, the module excludes their representations and employs a flexible fusion operator $ M(\cdot)$ to integrate the remaining single-modality representations $H \subseteq \{h_1, h_2, \ldots, h_n\}$: $\widehat{Y} = M(H).$

The Dice loss function \cite{p49} is employed to optimize the consistency between the fused prediction $\widehat{Y}$ and the target label $Y$, ensuring pixel-wise accuracy during training: 
\begin{equation}
\mathcal{L}_{Dice}(\widehat{Y}, Y) = 1 - \frac{2}{J} \sum_{j=1}^{J} \frac{\sum_{i=1}^{I} \widehat{Y}_{i,j} Y_{i,j}}{\sum_{i=1}^{I} \widehat{Y}_{i,j}^2 + \sum_{i=1}^{I} Y_{i,j}^2}, 
\end{equation}
where $I$ is the total number of voxels, $J$ is the number of classes, $\widehat{Y}_{i,j}$ is the predicted probability of voxel $i$ belonging to class $j$, and $Y_{i,j}$ is the one-hot encoded label.  

To address the challenges of medical image segmentation with missing modalities, a mutual information loss $\mathcal{L}_{\mathcal{MI}}$, inspired by Hinton's knowledge distillation \cite{p53}, is introduced to enhance the model's accuracy. The total loss is defined as: 
\begin{equation}
\mathcal{L}_{All} = \mathcal{L}_{Dice}(\widehat{Y}, Y) + \lambda_1 \mathcal{L}_{\mathcal{MI}} + \lambda_2 \mathcal{L}_{HD}(\widehat{Y}, Y),
\end{equation}
where $\lambda_1 > 0$ and $\lambda_2 > 0$ are weighting parameters. A summary of the algorithm is provided in Algorithm \ref{alg:spl}.

\newtheorem{hypothesis}{Hypothesis}
\subsection{Theoretical Analysis}
\begin{hypothesis}
Processing each MRI modality independently allows for the retention of unique, modality-specific information relevant to brain tumor segmentation, thereby enhancing segmentation accuracy even when certain modalities are absent.
\end{hypothesis}

\begin{hypothesis}
Maximizing mutual information for knowledge transfer between fully available and missing modalities compensates for segmentation accuracy losses due to missing data.
\end{hypothesis}

\begin{hypothesis}
Using Hölder divergence to measure the discrepancy between predicted and target segmentation maps provides robustness to non-symmetric data distributions, thereby enhancing segmentation performance in scenarios with missing modalities.
\end{hypothesis}


\newtheorem{theorem}{Theorem}

\begin{theorem}[\textbf{Preservation of Modality-Specific Features via Unique Representations}]
Let $x_i$ denote the input from modality $i$, and let $h_i = T(f_i(x_i); \theta)$ represent the single-modality representation obtained by applying a shared network backbone $T$ parameterized by $\theta$ to the modality-specific encoding $f_i(x_i)$. Assume that the function $T \circ f_i$ is injective for each modality $i$. Then, $h_i \neq h_j$ for all $i \neq j$, ensuring the preservation of modality-specific features.
\end{theorem}

\begin{proof}
Let $x_i$ be the input data from modality $i$, with $f_i(x_i)$ representing the modality-specific encoding function and $T(\cdot; \theta)$ the shared network backbone function. Then, $h_i = T(f_i(x_i); \theta)$ is the single-modality representation for modality $i$. Under the assumption that the composite function $T \circ f_i$ is injective for each modality $i$, different inputs from the same modality are mapped to unique representations, preserving modality-specific features \cite{pp1}. Furthermore, the encoding functions $f_i$ are distinct across modalities.

To show that $h_i \neq h_j$ for $i \neq j$, we consider that $f_i$ and $f_j$ are modality-specific encoders. For the same input $x$, $f_i(x) \neq f_j(x)$ because the encoders process the input differently according to modality-specific characteristics. Applying the shared network backbone $T$ to these distinct encodings results in different outputs: $h_i = T(f_i(x_i); \theta), \quad h_j = T(f_j(x_j); \theta).$ Since $f_i(x_i) \neq f_j(x_j)$ and $T$ is injective when composed with $f_i$ and $f_j$, it follows that $h_i \neq h_j$. Thus, the representations $h_i$ are unique to each modality, preserving modality-specific features.
\end{proof}

\begin{theorem}[\textbf{Enhancement of Segmentation Accuracy via Mutual Information Maximization}]
Let $d_f$ and $d_m$ denote the feature representations from the full-modality and missing-modality networks, respectively, and let $Y$ represent the segmentation labels. By maximizing the mutual information $MI(d_f; d_m)$, the conditional entropy $H(Y\mid d_m)$ is reduced, leading to improved segmentation accuracy in missing-modality scenarios.
\end{theorem}

\begin{proof}
Let $d_f$ be the feature representation derived from full modalities, $d_m$ the feature representation from missing modalities, and $Y$ the segmentation labels. Mutual information is defined as $MI(d_f; d_m) = H(d_f) - H(d_f \mid d_m)$ \cite{pp2}.

We assume that the segmentation labels $Y$ depend on the features $d_f$, while the features $d_m$ are intended to approximate $d_f$ to compensate for the missing modalities. Our objective is to demonstrate that maximizing $MI(d_f; d_m)$ decreases $H(Y \mid d_m)$, thereby improving segmentation accuracy.

According to the data processing inequality \cite{pp2}, $MI(Y; d_m) \leq MI(Y; d_f)$. By maximizing $MI(d_f; d_m)$, we make $d_m$ a closer approximation of $d_f$. Increasing $MI(d_f; d_m)$ effectively reduces $H(d_f \mid d_m)$, implying that $d_m$ retains more information about $d_f$. Since $Y$ depends on $d_f$, a reduction in $H(d_f \mid d_m)$ leads to a decrease in $H(Y \mid d_m)$, as the uncertainty in predicting $Y$ from $d_m$ diminishes. This lower conditional entropy $H(Y \mid d_m)$ corresponds to improved segmentation accuracy when $d_m$ is used in place of $d_f$.

Thus, by maximizing the mutual information between $d_f$ and $d_m$, we enhance the predictive power of $d_m$ with respect to $Y$, improving segmentation accuracy in scenarios with missing modalities.
\end{proof}

\begin{theorem}[\textbf{Suitability of Hölder Divergence for Non-Symmetric Distributions in Segmentation}]
\label{theorem_3}
Let $D_{\alpha}^{H}(p \| q) $ denote the Hölder divergence between two probability distributions $p(x)$ and $q(x)$, with conjugate exponents $\alpha > 1$ and $\beta = \frac{\alpha}{\alpha - 1}$. In segmentation tasks involving non-symmetric data distributions, Hölder divergence provides a robust measure of discrepancy that can capture asymmetries better than symmetric divergence measures.
\end{theorem}

\begin{proof}
By definition \cite{p37}, the Hölder Statistical Pseudo-Divergence (HPD) is given by
 \begin{equation}
     D_{\alpha}^{H}(p \| q)
= -\log\!\Bigl(\frac{\int p(x)\,q(x)\,dx}{ \|p\|_{\alpha}\,\|q\|_{\beta}}\Bigr),
 \end{equation}
where \(\| p \|_{\alpha} = \bigl(\int p(x)^\alpha\,dx\bigr)^{\!1/\alpha}\) and \(\beta = \alpha/(\alpha - 1)\).
HPD is sensitive to asymmetry between \(p\) and \(q\), making it well suited for cases where the ground truth and predicted distributions differ in a non-symmetric fashion (e.g., class imbalance or spatial variations).
Unlike KL divergence (which is asymmetric but less adaptable) or Jensen--Shannon divergence (which is symmetric), HPD allows tuning the degree of asymmetry through \(\alpha\).
This flexibility in handling distributional discrepancies makes it a robust choice for segmentation tasks involving non-symmetric distributions.measuring discrepancies in non-symmetric distributions encountered in segmentation tasks.
\end{proof}

\begin{theorem}[\textbf{Efficiency of Dynamic Sharing Framework in Reducing Inference Costs}]
Consider a model employing a dynamic sharing framework where a fusion operator $M(H)$ combines available modality representations $H \subseteq \{h_1, h_2, \ldots, h_n\}$. Assuming that the fusion operator $M$ compensates for missing modalities, the inference computational cost is reduced in proportion to the number of missing modalities, while segmentation accuracy remains within an acceptable margin.
\end{theorem}

\begin{proof}
Let $n$ denote the total number of modalities, $k=|H|$ the number of available modalities, $C_{\text{total}}$ the total computational cost, and $C_{\text{modality}}$ the computational cost per modality.

The computational cost scales linearly with the number of processed modalities, so $C_{\text{total}} = k \cdot C_{\text{modality}}$. This implies that reducing the number of modalities directly reduces the inference cost proportionally \cite{pp4}. The fusion operator $M$ is designed to integrate any subset of modality representations, preserving segmentation accuracy within an acceptable margin. Training strategies such as knowledge distillation \cite{p53} and mutual information maximization allow the model to adapt effectively to missing modalities.

Our goal is to show that reducing $k$ results in a proportional decrease in inference cost and that segmentation accuracy remains within a margin $\epsilon$. Since each modality contributes $C_{\text{modality}}$ to the computational cost, missing modalities reduce $C_{\text{total}}$ to $k \cdot C_{\text{modality}}$.

The model's architecture allows it to maintain performance despite missing inputs, as it learns to compensate during training. Empirical evidence supports that dynamic architectures can maintain segmentation accuracy with fewer modalities \cite{pp5}. The margin $\epsilon$ represents a minimal accuracy loss, acceptable in practical applications. Thus, the dynamic sharing framework effectively reduces inference costs in proportion to missing modalities while maintaining segmentation accuracy within an acceptable range, as verified through both theoretical analysis and empirical studies.
\end{proof}

\begin{theorem}[\textbf{Convergence of the Total Loss Function in Optimization}]
Let the total loss function be defined as 
\begin{equation}
\mathcal{L}_{\text{All}}(\theta) = \mathcal{L}_{\text{Dice}}(\theta) + \lambda_1 \mathcal{L}_{\text{MI}}(\theta) + \lambda_2 \mathcal{L}_{\text{HD}}(\theta), 
\end{equation}
where $\lambda_1, \lambda_2 \geq 0$ are weighting coefficients, and $\theta$ represents the model parameters. Under standard assumptions of smoothness and boundedness, gradient-based optimization algorithms will converge to a critical point of $\mathcal{L}_{\text{all}}$, leading to improved segmentation accuracy.
\end{theorem}

\begin{proof}
The loss components $\mathcal{L}_{\text{Dice}}(\theta)$, $\mathcal{L}_{\text{MI}}(\theta)$, and $\mathcal{L}_{\text{HD}}(\theta)$ are continuously differentiable with Lipschitz continuous gradients \cite{pp6}. Each loss function is bounded below, preventing divergence during optimization. A suitable gradient-based optimization method, such as stochastic gradient descent with diminishing learning rates, is applied \cite{pp7}.

The optimization algorithm guarantees that the total loss decreases with each iteration: 
\begin{equation}
\mathcal{L}_{\text{all}}(\theta_{t+1}) \leq \mathcal{L}_{\text{all}}(\theta_t) - \eta_t \| \nabla_\theta \mathcal{L}_{\text{all}}(\theta_t) \|^2,
\end{equation}
where $\eta_t$ is the learning rate at iteration $t$. Learning Rate Schedule: By setting $\eta_t$ such that $\sum_{t} \eta_t = \infty$ and $\sum_{t} \eta_t^2 < \infty$, convergence to a critical point is ensured \cite{pp7}.

Minimization of $\mathcal{L}_{\text{Dice}}$ directly improves the overlap between predictions and ground truth, boosting segmentation accuracy. Minimization of $\mathcal{L}_{\text{MI}}$ enhances the alignment between features from full and missing modalities, facilitating effective knowledge transfer \cite{pp8}. Minimization of $\mathcal{L}_{\text{HD}}$ reduces divergence between predicted and true distributions, refining model predictions. 

The overall reduction in $\mathcal{L}_{\text{all}}$ leads to improved segmentation performance due to the complementary effects of each loss component. Thus, under the given conditions, the optimization of $\mathcal{L}_{\text{all}}$ converges, resulting in enhanced segmentation accuracy as the model parameters $\theta$ approach a critical point that minimizes the total loss.
\end{proof}

\begin{table*}
\centering
\caption{Quantitative Evaluation of Segmentation Results (DSC $\uparrow$) on BraTS 2018. This table presents the quantitative results of segmentation performance, measured by the Dice Similarity Coefficient (DSC), on the BraTS 2018 dataset. The results provide a comparative evaluation of the effectiveness of different segmentation methods, where higher DSC values indicate better segmentation accuracy.}
\resizebox{\textwidth}{!}
{
\setlength{\tabcolsep}{3pt}
\begin{tabular}{c|c|ccccccccccccccc|c}
\toprule[1.0pt] 
Task & Methods & Fl & T2 & T1c & T1 & T2,Fl & T1c,Fl & T1c,T2 & T1,Fl & T1,T2 & T1,T1c &$\sim\mathrm{T}1$ &$\sim\mathrm{T}1c$ &$\sim\mathrm{T}2$ &$\sim\mathrm{Fl.}$  &Full & Avg. \\
\midrule[0.5pt]
\multirow{16}{*}{WT} 
&U-HVED \cite{pp20}	&82.1	&80.9	&62.4	&52.4	&87.5	&85.5	&82.7	&84.3	&82.2	&66.8	&88.6	&88.0 	&86.2	&83.3	&88.8	&80.1  \\
&ACN \cite{p10}	&87.3	&85.6	&80.5	&79.3	&87.8	&88.3	&86.4	&87.5	&85.5	&80.1	&88.3	&88.4	&89.0	&86.9	&89.2	&86.0  \\
&RFNet \cite{p44}	&85.8	&85.1	&73.6	&74.8	&89.3	&89.4	&85.6	&89.0	&85.4	&77.5	&90.4	&90.0	&89.9	&86.1	&90.6	&85.5 \\
&SMU-Net \cite{p5}	&87.5	&85.7	&80.3	&78.6	&87.9	&88.4	&86.1	&87.3	&85.6	&80.3	&88.2	&88.3	&88.2	&86.5	&88.9	&85.9  \\
&$\text{D}^2\text{-Net}$ \cite{pp24}	&84.2	&76.3	&42.8	&15.5	&87.9	&87.5	&84.1	&87.3	&80.1	&62.1	&88.8	&88.4	&87.7	&80.9	&88.8	&76.2  \\
&mmFormer \cite{p45}	&86.1	&81.2	&72.2	&67.5	&87.6	&87.3	&83.0	&87.1	&82.2	&74.4	&88.1	&87.8	&87.3	&82.7	&89.6	&82.9 \\
&$\text{M}^3\text{AE}$ \cite{p46}	&88.7	&84.8	&75.8	&74.4	&89.9	&89.7	&86.3	&89.0 	&86.7	&77.2	&90.2	&89.9	&88.9	&85.7	&90.1	&85.8 \\
&EMRM \cite{pp9}	&87.5	&86.5	&79.2	&78.7	&\textbf{90.8}	&90.1	&86.8	&89.9	&86.7	&79.6	&90.6	&90.7	&90.4	&88.1	&\textbf{91.3}	&87.1 \\
&MML-MM-SSF \cite{p9}	&88.8	&84.9	&75.0 	&74.8	&90.1	&90.1	&86.4	&89.9	&86.1	&78.9	&90.8	&90.4	&90.4	&86.5	&90.9	&86.3  \\
&QuMo \cite{pp12}	&89.3	&86.1	&77.6	&77.5	&90.4	&89.3	&86.8	&90.1	&86.6	&79.4	&89.7	&90.5	&89.2	&86.9	&89.7	&86.6  \\
&GSS \cite{pp13}	&87.7	&86.4	&78.5	&78.8	&89.9	&89.9	&87.9	&89.6	&87.5	&81.9	&90.7	&90.2	&90.3	&88.0 	&90.7	&87.2  \\
&MMCFormer \cite{p7}	&86.2	&84.1	&80.4	&78.6	&88.0	&88.7	&86.3	&87.7	&85.9	&80.2	&88.9	&88.5	&88.3	&86.8	&89.0	&85.8  \\
&$\text{M}^2\text{FTrans}$ \cite{pp15}	&87.2	&86.9	&77.8	&77.2	&89.2	&88.9	&88.1	&88.4	&87.5	&81.1	&89.8	&89.4	&89.0	&88.0	&89.7	&86.5  \\
&MTI \cite{p47}	&88.4	&86.6	&77.8	&78.7	&90.3	&89.5	&88.2	&89.7	&88.1	&81.8	&90.6	&89.7	&90.4	&88.4	&90.6	&87.3 \\
&GGDM \cite{p48}	&89.3	&87.0	&79.9	&75.9	&90.7	&90.6 	&88.6	&90.2	&88.2 	&81.1	&\textbf{91.3}	&\textbf{91.0} 	&90.5	&87.9	&91.0 	&87.6  \\
&OUR	&\textbf{89.8}	&\textbf{88.2}	&\textbf{80.5}	&\textbf{78.8}	&\textbf{90.8}	&\textbf{90.7}	&\textbf{89.3}	&\textbf{90.4}	&\textbf{88.7}	&\textbf{82.0}	&\textbf{91.3}	&\textbf{91.0}	&\textbf{90.9}	&\textbf{89.2}	&\textbf{91.3}	&\textbf{88.2}  \\
\midrule[0.5pt]

\multirow{16}{*}{TC} 
&U-HVED \cite{pp20}	&50.4	&54.1	&66.7	&37.2	&59.7	&72.9	&73.7	&55.3	&57.2	&69.7	&75.6	&61.5	&74.2	&75.3	&76.4	&64.0   \\
&ACN \cite{p10}	&67.7	&67.9	&84.2	&71.2	&71.6	&83.4	&84.4	&71.3	&73.3	&84.6	&82.9	&67.9	&84.3	&84.7	&85.2	&77.6  \\
&RFNet \cite{p44}	&62.6	&66.9	&80.3	&65.2	&71.8	&81.6	&82.4	&72.2	&71.1	&81.3	&82.6	&74.0	&82.3	&82.9	&82.9	&76.0 \\
&SMU-Net \cite{p5}	&71.8	&67.2	&84.1	&69.5	&71.2	&84.1	&85.0	&71.2	&73.5	&84.4	&82.5	&67.9	&84.2	&84.4	&87.3	&77.9  \\
&$\text{D}^2\text{-Net}$ \cite{pp24}	&47.3	&56.7	&65.1	&16.8	&62.6	&80.8	&80.3	&61.6	&63.2	&78.2	&80.7	&63.7	&80.9	&79.0	&80.1	&66.5  \\
&mmFormer \cite{p45}	&61.2	&64.2	&75.4	&56.6	&69.8	&77.9	&78.6	&65.9	&69.4	&78.6	&79.6	&71.5	&79.8	&80.4	&85.8	&73.0 \\
&$\text{M}^3\text{AE}$ \cite{p46}	&66.1	&69.4	&82.9	&66.4	&70.9	&84.4	&84.2	&70.8	&71.8	&83.4	&84.6	&72.7	&84.1	&84.4	&84.5	&77.4 \\
&EMRM \cite{pp9}	&68.8	&68.1	&84.6	&71.2	&73.7	&85.0 	&85.7	&73.5	&72.9	&84.4	&84.8	&75.6	&85.4	&86.6	&86.8	&79.1  \\
&MML-MM-SSF \cite{p9}	&69.6	&69.1	&81.4	&64.5	&72.9	&84.8	&84.1	&72.8	&71.4	&82.6	&85.7	&74.0	&85.4	&84.3	&85.8	&77.9  \\
&QuMo \cite{pp12}	&68.3	&70.5	&81.1	&65.8	&74.2	&83.3	&83.6	&74.2	&73.2	&82.0	&83.4	&74.8	&83.6	&83.3	&83.6	&77.7  \\
&GSS \cite{pp13}	&68.6	&69.4	&82.3	&67.5	&73.4	&83.7	&84.4	&73.8	&73.2	&83.7	&84.4	&75.4	&84.7	&84.6	&84.6	&78.2  \\
&$\text{M}^2\text{FTrans}$ \cite{pp15}	&69.9	&72.4	&82.6	&66.2	&75.4	&84.8	&85.2	&74.1	&74.5	&83.5	&85.3	&76.5	&85.3	&85.5	&85.7	&79.1  \\
&MMCFormer \cite{p7}	&70.0	&69.7	&86.6	&69.3	&72.4	&86.7	&86.7	&72.0	&74.0	&86.6	&86.6	&67.8	&86.6	&86.7	&87.4	&79.3  \\
&MTI \cite{p47}	&66.7	&68.8	&81.5	&65.6	&71.8	&84.8	&84.8	&72.0	&72.3	&83.5	&85.8	&74.1	&85.2	&85.8	&85.9	&77.9 \\
&GGDM \cite{p48}	&\textbf{77.3}	&76.3	&85.3	&58.1	&78.5	&\textbf{87.0} 	&\textbf{87.6}	&76.3	&76.8	&85.6	&\textbf{87.1}	&78.3	&86.5	&86.2	&85.8	&80.8 \\
&OUR	&76.2	&\textbf{77.6}	&\textbf{86.5}	&\textbf{72.6}	&\textbf{79.3}	&86.6	&87.2	&\textbf{78.6}	&\textbf{79.1}	&\textbf{86.9}	&\textbf{87.1}	&\textbf{80.1}	&\textbf{87.1}	&\textbf{87.4}	&\textbf{87.3}	&\textbf{82.6} \\
\midrule[0.5pt]

\multirow{16}{*}{ET} 
&U-HVED \cite{pp20}	&24.8	&30.8	&65.5	&13.7	&34.6	&70.3	&70.2	&24.2	&30.7	&67.0 	&71.2	&34.1	&71.1	&71.1	&71.7	&50.1  \\
&ACN \cite{p10}	&42.8	&43.0 	&78.1	&41.5	&46.0 	&77.5	&75.7	&43.7	&47.4	&75.2	&76.0 	&42.1	&76.2	&76.1	&77.1	&61.2  \\
&RFNet \cite{p44}	&35.5	&43.0	&67.7	&32.3	&45.4	&72.5	&70.6	&38.5	&42.9	&68.5	&73.1	&46.0	&71.1	&70.9	&71.4	&56.6 \\
&SMU-Net \cite{p5}	&46.1	&43.1	&78.3	&42.8	&46.0	&77.3	&75.7	&44.0	&47.7	&75.1	&75.4	&43.1	&76.2	&76.2	&79.3	&61.8  \\
&$\text{D}^2\text{-Net}$ \cite{pp24}	&8.1	&16.0	&66.3	&8.1	&17.4	&64.8	&68.7	&9.5	&16.5	&70.7	&66.4
&9.4	&65.7	&68.3	&68.4	&42.9  \\
&mmFormer \cite{p45}	&39.3	&43.1	&72.6	&32.5	&47.5	&75.1	&74.5	&43.0	&45.0	&74.0	&75.7	&47.7	&75.5	&74.8	&77.6	&59.9 \\
&$\text{M}^3\text{AE}$ \cite{p46}	&35.6	&47.6	&73.7	&37.1	&45.4	&75.0	&75.3	&41.2	&48.7	&74.7	&73.8	&44.8	&74.0 	&75.4	&75.5	&59.9 \\
&EMRM \cite{pp9}	&43.9	&44.9	&78.1	&41.1	&48.6	&77.3	&77.2	&45.8	&46.3	&77.6	&77.9	&48.2	&76.0 	&76.7	&78.3	&62.5  \\
&MML-MM-SSF \cite{p9}	&45.1	&46.2	&75.8	&42.6	&47.2	&77.8	&78.6	&44.8	&49.9	&78.3	&78.2	&46.6	&78.5	&78.4	&78.5	&63.1  \\
&QuMo \cite{pp12}	&44.3	&42.2	&75.6	&40.7	&48.6	&77.3	&75.0 	&49.0	&46.4	&76.0	&77.1	&52.6	&77.5	&76.5	&77.1	&62.4  \\
&GSS \cite{pp13}	&42.9	&45.8	&77.1	&42.4	&48.6	&77.8	&79.4	&47.3	&49.0 	&77.9	&78.7	&50.2	&78.4	&78.5	&78.3	&63.5  \\
&$\text{M}^2\text{FTrans}$ \cite{pp15}	&38.0 	&46.4	&78.9	&37.2	&49.1	&82.1	&80.9	&43.5	&47.2	&80.8	&80.6	&49.8	&82.2	&80.8	&80.6	&63.9  \\
&MMCFormer \cite{p7}	&49.6	&50.7	&79.0	&42.9	&48.3	&79.5	&79.1	&48.1	&46.9	&78.7	&79.1	&50.7	&78.1	&79.2	&80.1	&64.7  \\
&MTI \cite{p47}	&40.5	&41.4	&75.7	&44.5	&48.3	&76.8	&77.7	&44.4	&47.7	&77.1	&76.6	&50.0 	&77.4	&78.5	&80.4	&62.5 \\
&GGDM \cite{p48}	&47.4	&\textbf{53.4}	&81.6	&34.7	&55.2	&82.0	&82.6	&51.1	&54.7	&82.0	&82.1	&56.0	&82.2	&82.8	&82.1	&67.6 \\
&OUR	&\textbf{48.6}	&52.9	&\textbf{82.4}	&\textbf{48.7}	&\textbf{55.8}	&\textbf{83.0}	&\textbf{83.2}	&\textbf{53.8}	&\textbf{56.9}	&\textbf{82.8}	&\textbf{83.7}	&\textbf{58.4}	&\textbf{83.2}	&\textbf{83.5}	&\textbf{84.1}	&\textbf{69.4} \\
\bottomrule[1.0pt]
\end{tabular}
}
\label{tab-2}
\end{table*}

\section{Experiments and Analysis}
\subsection{Datasets and Evaluation Metrics}	
To improve the model's logical coherence, result reliability, and algorithm robustness in brain tumor segmentation, this study utilizes the BraTS 2018 and BraTS 2020 datasets\cite{p38}. These datasets are widely recognized in the field of medical imaging for multi-classification and segmentation tasks. They are extensive collections of multi-modal MRI scans (T1, T1Gd, T2, and FLAIR) from patients with high-grade and low-grade gliomas. Expertly annotated, these datasets mark tumor subregions like the enhancing tumor, peritumoral edema, and necrotic core. They are essential for advancing and validating automated brain tumor segmentation algorithms. To evaluate the effectiveness of our method, the Dice Similarity Coefficient (DSC)\cite{p49}, 
\begin{equation}
\begin{aligned}
\mathrm{Dice}(P,G)=\frac{2\times|P\cap G|}{|P|+|G|}
\end{aligned}
\end{equation}
is used, which is a common performance metric in medical image analysis. The DSC measures the overlap between the model's output ($P$) and the ground truth ($G$). A higher Dice coefficient indicates better predictive performance.

\subsection{Training Details}	
In this study, a PyTorch-based framework \cite{p43} (version 2.3.0) is utilized for training all models on a server equipped with dual NVIDIA RTX A6000 GPUs. The standard 3D U-Net architecture \cite{p41} is adopted, featuring a single encoder-decoder parallel processing structure that incorporates residual blocks and group normalization techniques. During training, a batch size of 8 is set, and the Adam optimizer \cite{p42} is employed to update model parameters, starting with a learning rate of 0.0008 and a weight decay of 0.00001. Training is conducted for 600 epochs to ensure comprehensive learning and performance optimization. Post-training, thorough testing of the model is conducted under all possible channel dropout configurations.

\subsubsection{Compare Experimental Models}	
The comparative experimental models employed in this study are U-HVED \cite{pp20}, ACN \cite{p10}, RFNet \cite{p44}, SMU-Net \cite{p5}, $\text{D}^2\text{-Net}$ \cite{pp24}, mmFormer \cite{p45}, $\text{M}^3\text{AE}$ \cite{p46}, EMRM \cite{pp9}, MML-MM-SSF \cite{p9}, MAVP \cite{pp11}, QuMo \cite{pp12}, GSS \cite{pp13}, SFusion \cite{pp14}, MMCFormer \cite{p7}, $\text{M}^2\text{FTrans}$ \cite{pp15},  MTI \cite{p47}, and GGDM \cite{p48}. Each model makes unique contributions to the field of missing modality segmentation, as outlined below. The results for U-HVED, ACN, RFNet, SMU-Net, $\text{D}^2\text{-Net}$, mmFormer, $\text{M}^3\text{AE}$, EMRM, MML-MM-SSF, MAVP, QuMo, GSS, SFusion, MMCFormer, $\text{M}^2\text{FTrans}$ and MTI are sourced from their respective original research papers, all adhering to the same experimental configuration as RFNet. Additionally, the GGMD method is tested, according to authors' code.  These models are outlined below:  \textbf{1. U-HVED (Dorent et al., MICCAI 2019) \cite{pp20}}: This work introduced a hetero-modal variational encoder-decoder framework capable of simultaneously completing missing modalities and performing segmentation tasks on multimodal data. \textbf{2. ACN (Wang et al., MICCAI 2021) \cite{p10}}: The Adversarial Co-Training Network (ACN) was developed for brain tumor segmentation, specifically addressing scenarios with missing imaging modalities. \textbf{3. RFNet (Ding et al., ICCV 2021) \cite{p44}}: RFNet is a region-aware fusion network designed for the segmentation of brain tumors in scenarios with incomplete multi-modal data. \textbf{4. SMU-Net (Azad et al., ICMIDL 2022) \cite{p5}}: SMU-Net, a style-matching U-Net architecture, was proposed to effectively segment brain tumors even in the absence of certain imaging modalities. \textbf{5. $\text{D}^2\text{-Net}$ (Yang et al., TMI 2022) \cite{pp24}}: The Dual Disentanglement Network (D2-Net) was designed to handle brain tumor segmentation challenges when imaging modalities are incomplete. \textbf{6. mmFormer (Zhang et al., MICCAI 2022) \cite{p45}}: mmFormer is a multimodal medical transformer developed to improve brain tumor segmentation in scenarios involving incomplete multimodal data. \textbf{7. $\text{M}^3\text{AE}$ (Liu et al., AAAI 2023) \cite{p46}}: M3AE is a multimodal representation learning approach for brain tumor segmentation that effectively handles missing modalities. \textbf{8. EMRM (Konwer et al., ICCV 2023) \cite{pp9}}: EMRM utilized a meta-learning approach to enhance modality-agnostic representations, improving brain tumor segmentation across diverse imaging modalities. \textbf{9. MML-MM-SSF (Wang et al., CVPR 2023) \cite{p9}}: MML-MM-SSF is a shared-specific feature modeling method tailored for multi-modal learning, effectively managing scenarios with missing modalities. \textbf{10. MAVP (Qiu et al., ACM-MM 2023) \cite{pp11}}: MAVP is a modality-aware visual prompting technique for brain tumor segmentation, specifically designed to address incomplete multi-modal contexts in medical imaging. \textbf{11. QuMo (Chen et al., MICCAI 2023) \cite{pp12}}: This method introduced a query re-training approach for modality-agnostic brain tumor segmentation, enhancing segmentation accuracy even with missing data modalities. \textbf{12. GSS (Qiu et al., ICCV 2023) \cite{pp13}}: GSS leveraged a category-aware group self-support learning method for brain tumor segmentation in incomplete multi-modal scenarios, improving segmentation accuracy by harnessing the mutual support among available modalities. \textbf{13. SFusion (Liu et al., MICCAI 2023) \cite{pp14}}: SFusion introduced a self-attention-based multimodal fusion block that effectively integrates multiple input modalities into a single unified representation, boosting performance in complex tasks. \textbf{14. MMCFormer (Karimijafarbigloo et al., MIDL 2023) \cite{p7}}: A novel transformer-based approach, MMCFormer, was proposed to compensate for missing imaging modalities, enhancing the accuracy of brain tumor segmentation. \textbf{15. $\text{M}^2\text{FTrans}$ (Shi et al., JBHI 2024) \cite{pp15}}: MFTrans presented a Modality-Masked Fusion Transformer developed for robust brain tumor segmentation in scenarios with incomplete modality data. \textbf{16. MTI (Ting and Liu, JBHI 2024) \cite{p47}}: MTI is a multimodal transformer designed to enhance brain tumor segmentation using incomplete MRI data. \textbf{17. GGMD (Wang et al., AAAI 2024) \cite{p48}}: GGMD is a method designed to enhance robustness in brain tumor segmentation when handling missing modalities.

Tables \ref{tab-2}--\ref{tab-3} showcase the performance of our research method on the BraTS 2018 and BraTS 2020 datasets, benchmarked against five state-of-the-art brain tumor segmentation techniques, and the symbol $ \sim(\cdot) $ in Tables \ref{tab-2}--\ref{tab-3} denotes the amissing of a specific modality, with optimal performance results highlighted in black across different tumor types. The results highlight our method's superior performance across all three evaluated tumor regions—Whole Tumor (WT), Tumor Core (TC), and Enhancing Tumor (ET)—achieving the highest average Dice Similarity Coefficient (DSC). Specifically, Table \ref{tab-2} demonstrates improvements of 0.6\% in WT, 1.8\% in TC, and 1.8\% in ET regions compare to existing state-of-the-art methods on the BraTS 2018 dataset. Similarly, Table \ref{tab-3} shows enhancements of 0.8\% in WT, 0.7\% in TC, and 0.9\% in ET regions on the BraTS 2020 dataset.

\begin{figure}
	\centering
	\includegraphics[width=1.0\linewidth]{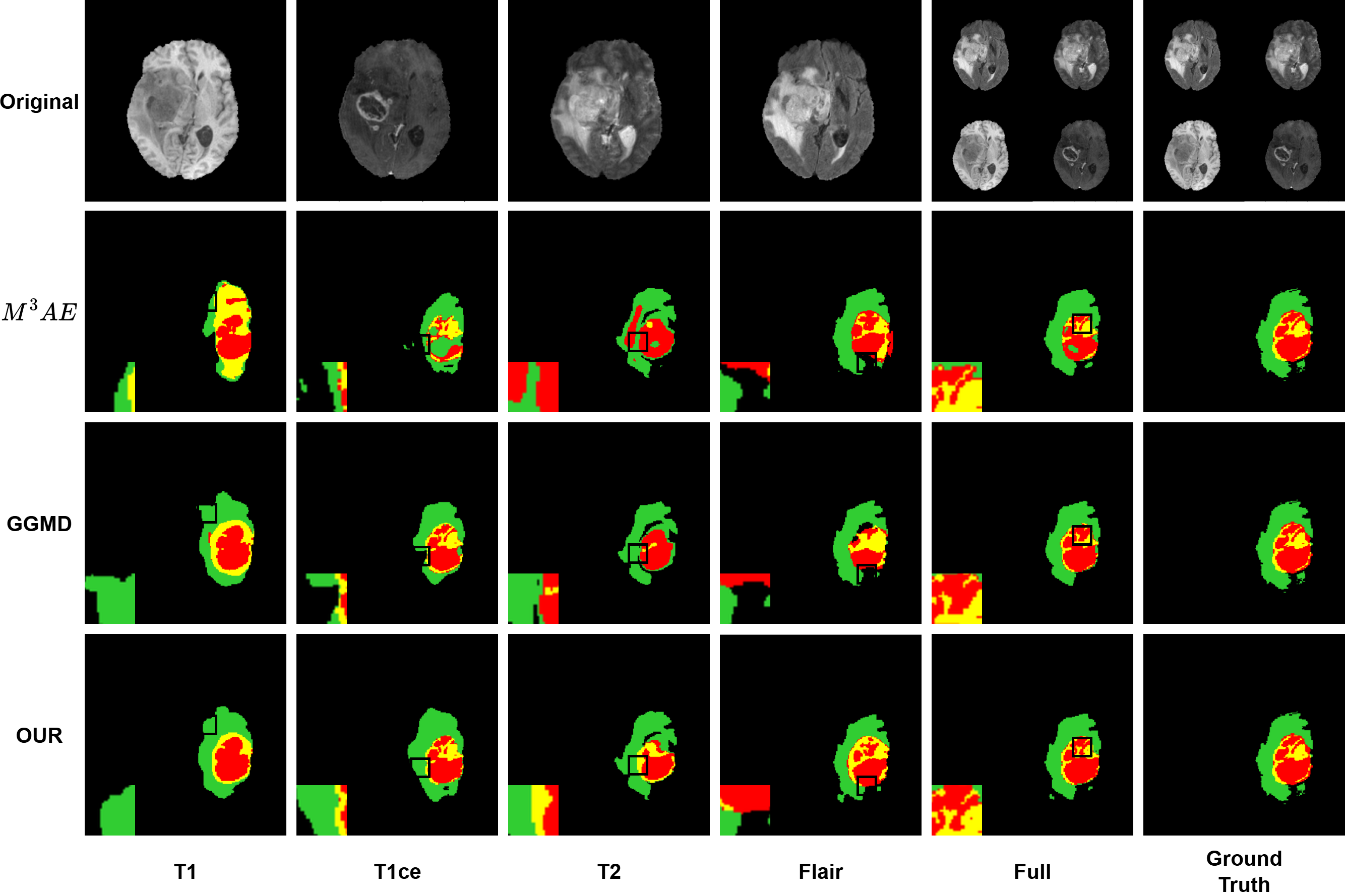}
	\caption{Segmentation results of three models on the BraTS 2018 dataset using various modality inputs. The second row presents reproduced results from the $M^3AE$ model; the third row shows reproduced results from the GGMD model; and the fourth row displays results from our proposed method. Each column represents a different input configuration: the first four columns correspond to single-modality inputs (T1, T1ce, T2, and Flair, respectively), the fifth column illustrates results using all four modalities simultaneously, and the last column provides the ground truth segmentation for reference.}
	\label{fig:_2}
\end{figure}

Our method excels notably in scenarios with missing multimodal data, achieving substantial gains of 2.0\% to 6.0\% in Dice coefficients. For instance, in Table \ref{tab-2}, when only the T1 modality is available, our method outperforms other advanced algorithms by 0.1\%, 6.2\%, and 4.2\% in WT, TC, and ET regions, respectively. Correspondingly, Table \ref{tab-3} indicates improvements of 2.4\%, 2.9\%, and 3.0\% under similar conditions.

These findings underscore the robustness of our method in maintaining efficient segmentation performance despite significant deficits in multimodal data. Furthermore, Fig. \ref{fig:_2} compares our method with other advanced techniques like M3AE and GGMD under single-modal and full-modal input conditions. These segmentation results validate our conclusions from the evaluation metrics, demonstrating that our method surpasses other advanced technologies in handling multimodal data, especially when only a single modality is available. This performance advantage is particularly notable under single-modal input conditions.
 \begin{table*}
\centering
\caption{Quantitative Evaluation of Segmentation Results (DSC $\uparrow$) on BraTS 2020. This table provides a quantitative assessment of the segmentation performance on the BraTS 2020 dataset, measured using the Dice Similarity Coefficient (DSC). An upward arrow ($\uparrow$) indicates that higher DSC values correspond to better segmentation accuracy, allowing a clear comparison of the effectiveness of different models or approaches on this dataset.}
\resizebox{\textwidth}{!}
{
\setlength{\tabcolsep}{3pt}
\begin{tabular}{c|c|ccccccccccccccc|c}
\toprule[1.0pt] 
Task & Methods & Fl & T2 & T1c & T1 & T2,Fl & T1c,Fl & T1c,T2 & T1,Fl & T1,T2 & T1,T1c &$\sim\mathrm{T}1$ &$\sim\mathrm{T}1c$ &$\sim\mathrm{T}2$ &$\sim\mathrm{Fl.}$ & Full & Avg. \\
\midrule[0.5pt]
\multirow{10}{*}{WT} 
&RFNet \cite{p44}	&87.3	&86.1	&76.8	&77.2	&89.9	&89.9	&87.7	&89.7	&87.7	&81.1	&90.7	&90.6	&90.7	&88.3	&91.1	&87.0 \\
&mmFormer \cite{p7}	&86.5	&85.5	&78.0	&76.2	&89.4	&89.3	&87.5	&88.7	&86.9	&80.7	&90.4	&89.8	&89.7	&87.6	&90.5	&86.4 \\
&$\text{M}^3\text{AE}$ \cite{p46}	&86.5	&86.1	&73.9	&76.7	&89.3	&89.5	&87.4	&89.4	&87.2	&78.1	&90.2	&90.4	&90.0	&88.6	&90.6	&86.3 \\
&MAVP \cite{pp11}	&86.9	&86.7	&79.5	&79.5	&90.1	&90.0	&88.4	&89.8	&87.9	&83.1	&90.8	&90.6	&90.6	&88.9	&91.0	&87.6  \\
&GSS \cite{pp13}	&88.1	&87.6	&80.1	&79.8	&90.4	&90.6	&88.7	&90.1	&88.7	&83.1	&91.3	&91.1	&91.1	&89.0 	&91.6	&88.1  \\
&SFusion \cite{pp14}	&84.0 	&79.6	&69.8	&69.1	&87.5	&87.3	&84.3	&86.4	&81.8	&75.3	&89.1	&87.6	&88.0	&84.6	&88.9	&82.9  \\
&$\text{M}^2\text{FTrans}$ \cite{pp15}	&88.7	&87.2	&78.8	&79.2	&90.6	&90.4	&88.7	&90.3	&88.3	&82.4	&91.2	&90.9	&91.0	&89.0 	&91.4	&87.9  \\
&MTI \cite{p47}	&89.1	&86.5	&77.4	&78.1	&90.5	&90.0	&88.4	&89.9	&88.0	&81.2	&90.6	&90.3	&90.7	&88.7	&90.6	&87.3 \\
&GGDM \cite{p48} &91.0	&88.3	&80.6	&77.4	&92.1	&91.9	&89.8	&91.6	&89.3	&82.7	&92.3	&92.1	&91.6	&89.2	&92.0 	&88.8 \\
&OUR    &\textbf{91.9}	&\textbf{89.2}	&\textbf{81.5}	&\textbf{80.5}	&\textbf{92.5}	&\textbf{92.4}	&\textbf{90.0}	&\textbf{92.1}&\textbf{89.8}	&\textbf{83.2}	&\textbf{92.8}	&\textbf{92.6}	&\textbf{92.3}&\textbf{89.9}	&\textbf{92.7}	&\textbf{89.6}	\\
\midrule[0.5pt]

\multirow{10}{*}{TC} 
&RFNet \cite{p44}	&69.2	&71.0	 &81.5	&66.0	&74.1	&84.7	&83.5	&73.1	&73.1	&83.4	&85.0	&75.2	&85.1	&83.5	&85.2	&78.2 \\
&mmFormer \cite{p7}	&64.6	&63.3	&81.5	&63.2	&70.3	&83.7	&82.6	&71.7	&67.7	&82.8	&83.9	&72.4	&84.4	&79.0	&84.6	&75.7 \\ 
&$\text{M}^3\text{AE}$ \cite{p46}	&68.0	&70.3	&81.4	&66.0	&75.0	&82.0	&83.0	&73.8	&72.5	&82.4	&83.1	&75.1	&82.4	&84.1	&84.4	&77.6 \\
&MAVP \cite{pp11}	&69.9	&71.0	&84.3	&67.7	&74.5	&86.6	&86.3	&74.4	&72.9	&85.8	&86.4	&75.8	&86.7	&86.5	&86.4	&79.7  \\
&GSS \cite{pp13}	&72.5	&72.3	&83.4	&66.4	&75.7	&86.0	&84.6	&73.7	&73.4	&83.2	&86.0	&75.7	&85.8	&84.3	&85.8	&79.3  \\
&SFusion \cite{pp14}	&52.8	&62.0	&75.6	&53.9	&66.4	&79.1	&81.5	&62.3	&66.7	&80.4	&82.1	&68.8	&82.0	&82.3	&82.2	&71.9  \\
&$\text{M}^2\text{FTrans}$ \cite{pp15}	&72.2	&72.3	&81.9	&66.8	&75.4	&85.5	&84.6	&74.4	&73.6	&83.7	&85.3	&76.1	&85.8	&84.9	&85.4	&79.2  \\
&MTI \cite{p47}	&69.3	&71.5	&83.4	&66.8	&75.5	&85.6	&86.4	&73.9	&73.3	&85.2	&86.4	&75.9	&86.5	&86.5	&87.4	&79.6 \\
&GGDM \cite{p48}	&77.0	&79.1	&\textbf{87.6}	&70.5	&81.6	&88.0	&88.5	&80.3	&\textbf{81.1}	&88.0	&88.1	&82.4	&87.9	&88.4	&88.0 	&83.8 \\
&OUR	&\textbf{78.6}	&\textbf{79.8}	&87.5	&\textbf{73.4}	&\textbf{82.5}	&\textbf{88.9}	&\textbf{88.6}	&\textbf{81.2}	&80.8	&\textbf{88.1}	&\textbf{88.8}	&\textbf{82.5}	&\textbf{88.9}	&\textbf{88.6}	&\textbf{88.8}	&\textbf{84.5}  \\
\midrule[0.5pt]

\multirow{10}{*}{ET} 
&RFNet \cite{p44}	&38.2	&46.3	&74.9	&37.3	&49.3	&76.7	&75.9	&41.0	&45.7	&78.0	&77.1	&49.9	&76.8	&77.0	&78.0	&61.5 \\
&mmFormer \cite{p7}	&36.6	&49.0	&78.3	&37.6	&49.0	&79.4	&77.2	&42.9	&49.1	&81.7	&78.7	&50.0	&80.6	&68.3	&79.9	&62.6 \\
&$\text{M}^3\text{AE}$ \cite{p46}	&40.5	&46.0	&72.4	&39.9	&47.3	&74.7	&76.8	&43.2	&46.6	&75.4	&77.1	&48.2	&75.9	&77.4	&78.0	&61.3 \\
&MAVP \cite{pp11}	&42.8	&47.2	&81.4	&39.1	&49.5	&81.2	&80.1	&46.8	&47.3	&81.7	&80.0	&51.1	&81.8	&82.1	&81.0   &64.9  \\
&GSS \cite{pp13}	&42.3	&51.3	&78.6	&39.7	&54.4	&81.0	&80.5	&49.5	&53.1	&80.8	&81.1	&53.9	&82.3	&81.2	&83.0	&66.2  \\
&SFusion \cite{pp14}	&34.4	&35.9	&71.9	&29.7	&41.5	&75.8	&74.7	&38.2	&40.4	&74.9	&74.9	&43.5	&75.4	&74.8	&73.8	&57.3  \\
&$\text{M}^2\text{FTrans}$ \cite{pp15}	&43.4	&51.5	&82.6	&40.9	&53.9	&83.1	&82.4	&47.0	&49.9	&83.8	&81.2	&53.3	&84.1	&82.4	&82.2	&66.8  \\
&MTI \cite{p47}	&43.6	&45.6	&78.9	&41.3	&48.7	&81.8	&81.7	&48.2	&50.0	&79.2	&81.0	&52.5	&81.8	&78.5	&81.6	&65.0 \\
&GGDM \cite{p48}	&49.8	&52.5	&84.2	&39.7	&56.5	&84.6	&84.5	&54.6	&\textbf{55.3}	&84.2	&84.2	&\textbf{58.6}	&84.3	&\textbf{84.3}	&84.1	&69.4  \\
&OUR	&\textbf{51.9}	&\textbf{54.6}	&\textbf{84.6}	&\textbf{44.3}	&\textbf{57.7}	&\textbf{84.9}	&\textbf{84.8}	&\textbf{55.4}	&55.1	&\textbf{84.3}	&\textbf{84.9}	&\textbf{58.6}	&\textbf{84.4}	&\textbf{84.3}	&\textbf{84.3}	&\textbf{70.3}  \\
\bottomrule[1.0pt]
\end{tabular}
}
\label{tab-3}
\end{table*}

\begin{table}
\centering
\caption{Performance Comparison of Hölder Divergence, $f$-Divergences, and Standard Loss Functions (MSE, BCE) on BraTS 2018 Dataset. This table provides a systematic quantitative evaluation of these metrics within the dataset context.}
\setlength{\tabcolsep}{10pt}
\begin{tabular}{c|ccc|c}
\toprule[1.0pt]
\multicolumn{1}{c}{\textbf{Methods}} & \multicolumn{3}{c}{\textbf{Dice}} & \multicolumn{1}{c}{\textbf{}} \\
\midrule[0.5pt]
\textbf{$f$-divergence }     &\textbf{WT} & \textbf{TC} & \textbf{ET}  &\textbf{Avg.} \\
\midrule[0.5pt]
\textbf{Total Variation} \cite{pp43}         &67.2  &1.9   &0.9   &23.3   \\
\textbf{Squared Hellinger } \cite{pp40}       &85.3  &75.4  &60.1  &73.6  \\
\textbf{Kullback-Leibler} \cite{pp41}               &84.5  &76.2  &61.4  &74.0  \\
\textbf{Neyman $\chi^2$} \cite{p50}        &83.4  &75.1  &59.9  &72.8  \\
\textbf{Jensen-Shannon } \cite{pp42}         &84.6  &76.5  &59.8  &73.6  \\
\textbf{MSE Loss} \cite{pp38}         &85.3  &80.7  &65.2  &77.1  \\
\textbf{BCE Loss} \cite{pp39}         &85.5  &80.9  &64.7  &77.0  \\
\textbf{Hölder} \cite{p37}                 &$\textbf{88.2}$  &$\textbf{82.6}$  &$\textbf{69.4}$  &$\textbf{80.1}$  \\
\bottomrule[1.0pt]
\end{tabular}
\label{tab-4}
\end{table}

\subsubsection{Exploration of the Superiority of Hölder Divergence}	
To explore the superiority of Hölder divergence, this study conduct experimental comparisons using Hölder divergence and other $f$-divergences, including Total Variation \cite{pp43}, Squared Hellinger \cite{pp40}, Kullback-Leibler \cite{pp41}, Neyman $\chi^2$ \cite{p50}, Jensen-Shannon divergence \cite{pp42}, mean squared error (MSE) \cite{pp38}, and binary cross-entropy (BCE) \cite{pp39} loss functions. As shown in the table \ref{tab-4}, the average Dice coefficient of Hölder divergence reached 80.1\% after adjusting the hyperparameter $\alpha$, which is 3.0\% higher than the best-performing alternative methods. This significant performance advantage underscores the importance of Hölder divergence in improving model accuracy.

The experimental data consistently demonstrate that the application of Hölder divergence significantly enhances segmentation task performance, validating its effectiveness in the field of medical image processing. Additionally, our research reveals the critical role of Hölder divergence in enhancing knowledge distillation techniques to improve segmentation efficiency, providing valuable references and guidance for future research and development in related technologies. These findings not only deepen our understanding of the potential of Hölder divergence but also provide empirical evidence for optimizing deep learning models using this method.

\begin{table}
\centering
\caption{Quantitative evaluation results of the ablation study on the BraTS 2018 dataset. This table illustrates the impact of the parallel network framework (PNF) and individual model components on performance, highlighting their contributions to enhancing overall accuracy and robustness.}
\setlength{\tabcolsep}{5pt}
\begin{tabular}{cccc|cccc|c}
\toprule[1.0pt]
\multicolumn{4}{c}{\textbf{Methods}} & \multicolumn{4}{c}{\textbf{Number of Missing Modalities}} & \multicolumn{1}{c}{\textbf{}} \\
\midrule[0.5pt]
PNF  &$\mathcal{L}_{dice}$  &$\mathcal{L}_{MI}$  &$\mathcal{L}_{HD}$  &\textbf{3} & \textbf{2} & \textbf{1} & \textbf{0} & \textbf{Avg.} \\
\midrule[0.5pt]
$\times$   &$\checkmark$   &$\times$  &$\times$   &46.8  &59.9  &69.3  &74.0  &59.8  \\
$\checkmark$   &$\checkmark$   &$\times$  &$\times$  &60.2  &71.9  &77.1  &80.5  &70.7  \\
$\checkmark$   &$\checkmark$   &$\checkmark$   &$\times$ &72.4  &79.7  &83.9  &87.0  &79.4  \\
$\checkmark$   &$\checkmark$   &$\times$ &$\checkmark$  &72.9  &80.1  &84.3  &87.4  &79.8 \\
$\checkmark$   &$\checkmark$   &$\checkmark$   &$\checkmark$   & \textbf{73.6}  & \textbf{80.3} & \textbf{84.4}  & \textbf{87.6} & \textbf{80.1} \\
\bottomrule[1.0pt]
\end{tabular}
\label{tab-6}
\end{table}

\begin{table}[htbp]
\centering
\caption{Quantitative comparison of training times on the BraTS 2018 dataset. The table compares the training runtime of our proposed method against several state-of-the-art methods, highlighting the efficiency advantage of our approach.}
\label{tab:training_time}
\setlength{\tabcolsep}{32pt}
\begin{tabular}{lc}
\toprule
\textbf{Method} & \textbf{Training (h)} \\
\midrule
ACN~\cite{p10} & 226.3 \\
SMU-Net~\cite{p5} & 154.4 \\
ProtoKD~\cite{p6} & 237.8 \\
MML-MM-SSF~\cite{p9} & 425.1 \\
MMCFormer~\cite{p7} & 130.5 \\
M\textsuperscript{2}FTrans~\cite{pp15} & 100.7 \\
GGDM~\cite{p48} & 30.1 \\
\textbf{OUR} & \textbf{22.5} \\
\bottomrule
\end{tabular}
\label{tab-7}
\end{table}
\begin{table}[htbp]
\centering
\caption{Quantitative Evaluation of Parallel Computing Efficiency on the BraTS 2018 Dataset. This table provides a quantitative assessment of parallel computing efficiency on the BraTS 2018 dataset, reporting metrics such as speedup ratios and parallel efficiency across different numbers of computational cores or threads. The results demonstrate the effectiveness of the proposed method in leveraging parallelism to accelerate computations, and highlight its scalability and performance improvements achieved through parallel implementation.}
\setlength{\tabcolsep}{10pt}
\begin{tabular}{c|cccc}  
\toprule[1.0pt]
\multirow{2}{*}{\textbf{Evaluation Criteria}} & \multicolumn{4}{c}{\textbf{Number of GPUs}} \\  
& \textbf{1} & \textbf{2} & \textbf{3} & \textbf{4} \\
\midrule[0.5pt]
\textbf{Training (h)}     &30.1  &22.5  &\textbf{19.7}  &25.5   \\
\textbf{Speedup}          &1.0     &1.34   &\textbf{1.53}   &1.18   \\
\textbf{Parallel Efficiency (\%)}     &100.0  &\textbf{66.9}  &50.9  &29.5  \\
\bottomrule[1.0pt]
\end{tabular}
\label{tab-8}
\end{table}

\subsection{Exploring the Impact of Hölder Conjugate Exponents on Experimental Results}
To further investigate the impact of Hölder conjugate exponents on experimental results, we explore various Hölder conjugate exponents. As shown in Table \ref{tab-5}, this study compares the performance under different Hölder hyperparameters ($\alpha$), KLD, and without the application of knowledge distillation. The experimental results indicate that when the Hölder divergence hyperparameter $\alpha = 1.1$, the performance improves by an average of 0.7\% compared to the case without knowledge distillation and by 6.1\% compared to KL divergence. This result underscores the crucial role of selecting an appropriate Hölder conjugate exponent ($\alpha$) in significantly enhancing model performance. 

Our findings clearly demonstrate the critical role of Hölder divergence in enhancing knowledge distillation techniques to improve segmentation task efficiency. Throughout our experiments, we consistently observe that setting the Hölder conjugate exponent to $\alpha = 1.1$ markedly improves the model's segmentation performance, further validating the effectiveness of Hölder divergence.

\subsection{Training Time Comparison Analysis}

In this section, we perform experiments on the BraTS 2018 dataset to comprehensively compare the training efficiency of our proposed method with several state-of-the-art methods, including ACN \cite{p10}, SMU-Net \cite{p5}, ProtoKD \cite{p6}, MML-MM-SSF \cite{p9}, MMCFormer \cite{p7}, $\text{M}^2\text{FTrans}$ \cite{pp15}, and GGMD \cite{p48}. Table 6 presents these experimental results, demonstrating that our method significantly outperforms existing state-of-the-art approaches in terms of training time. Specifically, our method achieves a training time of only 22.5 hours, representing a 33.8\% improvement in efficiency compared to the next-best algorithm. In contrast, the training times of existing methods are substantially higher: ACN requires 226.3 hours, SMU-Net 154.4 hours, ProtoKD 237.8 hours, MML-MM-SSF 425.1 hours, MMCFormer 130.5 hours, $\text{M}^2\text{FTrans}$ 100.7 hours, and GGMD 30.1 hours. These results highlight that our algorithm substantially reduces computational cost for the same task, offering a considerable advantage in practical scenarios. Through this comparative analysis, we clearly illustrate the superiority of our approach in terms of training efficiency, especially in scenarios that demand rapid processing and deployment.

\subsection{Quantitative Assessment of Parallel Computing Efficiency}

To evaluate the impact of parallelization on computational performance, we perform training experiments on the BraTS 2018 dataset using 1 to 4 GPUs. Table~\ref{tab-8} summarizes the results in terms of training time, speedup, and parallel efficiency, reflecting both performance gains and scalability.

When training with a single GPU, the process takes 30.1 hours. Using 2 GPUs reduces the training time to 22.5 hours, resulting in a speedup of 1.34× and a parallel efficiency of 66.9\%. With 3 GPUs, the training time further decreases to 19.7 hours, corresponding to a speedup of 1.53× and an efficiency of 50.9\%. However, with 4 GPUs, the training time increases to 25.5 hours, yielding a lower speedup of 1.18× and a parallel efficiency of just 29.5\%.

Fig.~\ref{fig:gpu-speedup} illustrates the non-linear relationship between the number of GPUs and the actual speedup achieved. The diminishing returns arise from factors such as communication overhead, load imbalance, and resource contention. The decrease in efficiency with 4 GPUs indicates that the parallelization strategy does not scale effectively beyond 3 GPUs for this dataset and model.

This analysis underscores the importance of evaluating both speedup and parallel efficiency in multi-GPU training. Optimizing workload distribution and communication strategies is essential to fully leverage the potential of parallel computing.

\subsection{Parameter Selection Analysis}

In our proposed method, selecting key parameters is essential for achieving robust segmentation performance under missing-modality conditions. This section provides a mathematical analysis of each parameter, including the Hölder divergence exponent $\alpha$, the mutual information knowledge transfer weighting parameter $\gamma_k$, and the learning rate $lr$ and batch size $b$, to justify their selection based on mathematical properties and ensure optimal model performance.

\subsubsection{Hölder Divergence Exponent $\alpha$}

The Hölder divergence between two probability distributions $p(x)$ and $q(x)$ is defined for conjugate exponents $\alpha > 1$ and $\beta = \frac{\alpha}{\alpha - 1}$ as:
\begin{equation}
    D_{\alpha}^{H}(p \| q) = -\log\left( \frac{\int p(x) q(x) \, dx}{ \| p(x) \|_{\alpha} \| q(x) \|_{\beta} } \right),
\end{equation}
where the $L^\alpha$-norm is defined as $\| p(x) \|_{\alpha} = \left( \int p(x)^{\alpha} \, dx \right)^{1/\alpha}$.

The choice of $\alpha$ directly influences the divergence’s sensitivity to discrepancies between $p(x)$ and $q(x)$, especially in non-symmetric data distributions. When $\alpha = 1$, Hölder divergence reduces to KLD, which is less suitable for handling the non-symmetry common in medical image segmentation due to class imbalance and other factors.


\begin{figure}[t]
	\centering
	\includegraphics[width=1.0\linewidth]{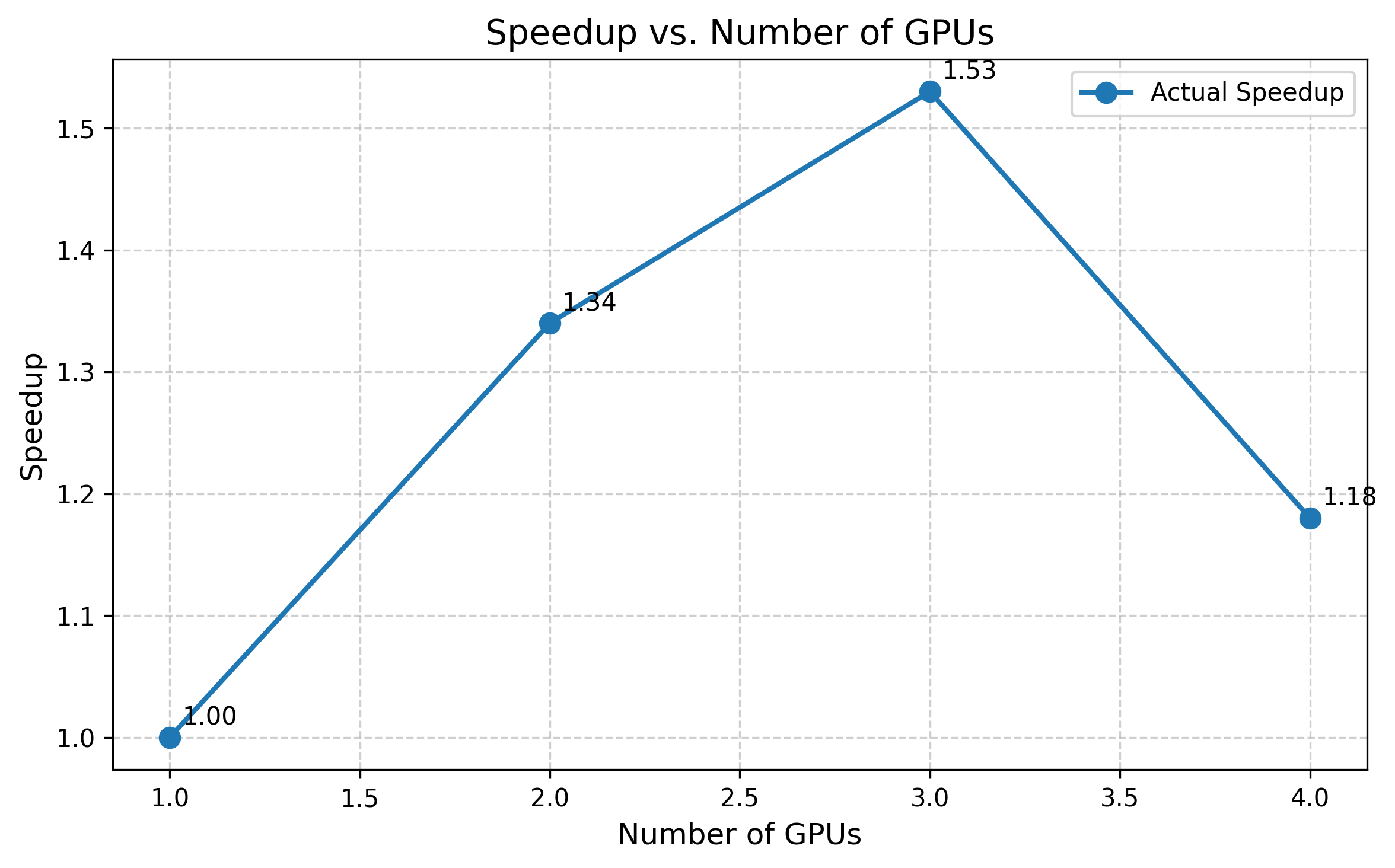}
	\caption{Analysis of the relationship between the number of GPUs and the observed computational speedup on the BraTS2018 dataset. This figure illustrates how actual speedup varies with increasing GPU count, demonstrating the parallel scalability of the proposed method. Both efficiency gains and potential scalability bottlenecks are highlighted as the number of GPUs increases.}
	\label{fig:gpu-speedup}
\end{figure}

For values of $\alpha$ slightly greater than 1, Hölder divergence becomes more responsive to significant differences in regions where $p(x)$ and $q(x)$ differ substantially. This effect results from the properties of the $L^\alpha$-norm, which gives more weight to higher discrepancies.

\begin{theorem}
Selecting $\alpha = 1.1$ in Hölder divergence provides a balance between sensitivity to distribution differences and robustness to noise, improving segmentation accuracy compared to both $\alpha = 1$ (KL divergence) and larger values of $\alpha$.
\end{theorem}

\begin{proof}
As $\alpha$ approaches 1, the divergence is close to KL divergence, capturing only small discrepancies and potentially neglecting significant differences. With increasing $\alpha$, the $L^\alpha$-norm becomes more sensitive to larger values in $p(x)$ and $q(x)$, effectively highlighting regions with high probability density. However, very large $\alpha$ values can lead to excessive sensitivity to noise.

Empirically, $\alpha = 1.1$ provides a moderate sensitivity increase that enhances segmentation accuracy by approximately 6.1\% over KL divergence (see Table~\ref{tab-5}). This value optimally balances capturing critical differences while maintaining robustness.
\end{proof}

\subsubsection{Mutual Information Knowledge Transfer Weighting Parameter $\gamma_k$}

In our knowledge transfer framework, we aim to maximize mutual information between features $d_f^{(k)}$ from the full-modality network and $d_m^{(k)}$ from the missing-modality network at each layer $k$, with the loss function defined as: 
\begin{equation}
\mathcal{L}_{\mathcal{MI}} = -\sum_{k=1}^{K} \gamma_k \mathbb{E}_{d_f^{(k)}, d_m^{(k)}} \left[ \log q\left( d_f^{(k)} \mid d_m^{(k)} \right) \right],
\end{equation}
where $q\left( d_f^{(k)} \mid d_m^{(k)} \right)$ approximates the conditional distribution $p\left( d_f^{(k)} \mid d_m^{(k)} \right)$, and $\gamma_k$ is a layer-specific weighting parameter.

We set $\gamma_k$ to increase with layer depth $k$, i.e., $\gamma_k = \gamma_0 \cdot k$, with $\gamma_0$ as a base weight, to emphasize the alignment of semantically rich, higher-level features.

\begin{theorem}
Increasing $\gamma_k$ with layer depth $k$ prioritizes the alignment of higher-level, semantically significant features, enhancing segmentation accuracy in missing-modality scenarios.
\end{theorem}

\begin{proof}
In deeper layers, features represent higher-level abstractions with more complex semantic information. Thus, aligning features in these layers has a greater impact on segmentation performance.

By assigning larger weights $\gamma_k$ to deeper layers, the loss function $\mathcal{L}_{\mathcal{MI}}$ prioritizes aligning these essential features. This emphasis reduces the conditional entropy $H(Y \mid d_m)$, enhancing segmentation accuracy.
\end{proof}

\subsubsection{Learning Rate $lr$ and Batch Size $b$}

The learning rate $lr$ and batch size $b$ are critical for convergence and stability in training. Using stochastic gradient descent, the model parameter update rule is: 
\begin{equation}
\theta^{(t+1)} = \theta^{(t)} - \eta \nabla_{\theta} \mathcal{L}_{\text{all}}(\theta^{(t)}),
\end{equation}
where $\eta$ is the effective step size.

We select $lr = 0.0008$ and $b = 8$ to balance convergence speed and stability.

Through analysis, we justify the selection of key parameters: 1. The Hölder divergence exponent $\alpha = 1.1$ achieves a balance between sensitivity to significant discrepancies and robustness to noise, enhancing segmentation accuracy. 2. The mutual information weighting parameter $\gamma_k$ increases with layer depth $k$, emphasizing the alignment of high-level features, which improves knowledge transfer effectiveness. 3. The learning rate $lr = 0.0008$ and batch size $b = 8$ ensure stable convergence to a local minimum, balancing convergence speed and stability.

These selections, based on mathematical properties of the divergence measures, optimization algorithms, and neural network architectures, contribute to the robustness and effectiveness of our model in handling missing-modality segmentation tasks.

\begin{table}
\centering
\caption{Exploring the Impact of Hölder Conjugate Exponents on Experimental Results Based on the Brats 2018 Dataset. This table illustrates how varying the Hölder conjugate exponents affects the experimental outcomes derived from the BraTS 2018 dataset. It highlights the relationships between different conjugate exponent values and their influence on the performance metrics within the context of the experiments.}
\setlength{\tabcolsep}{10pt}
\begin{tabular}{cc|ccc|c}
\toprule[1.0pt]
\multicolumn{2}{c}{\textbf{Methods}} & \multicolumn{3}{c}{\textbf{Dice}} & \multicolumn{1}{c}{\textbf{}} \\
\midrule[0.5pt]
\textbf{divergence }  &$\bm{\alpha}$    &\textbf{WT} & \textbf{TC} & \textbf{ET}  &\textbf{Avg.} \\
\midrule[0.5pt]
-                &-                &87.8  &82.9  &67.4  &79.4    \\
\textbf{KL}      &-                &84.5  &76.2  &61.4  &74.0  \\
\textbf{Hölder}  &1.05    &87.8  &82.6  &69.2  &79.9  \\
\textbf{Hölder}  &1.08    &88.2  &82.6  &68.7  &79.9  \\
\textbf{Hölder}  &\textbf{1.10}   &\textbf{88.2} &82.6  &\textbf{69.4}  &\textbf{80.1}  \\
\textbf{Hölder}  &1.15   &88.1  &82.8  &67.9  &79.6   \\
\textbf{Hölder}  &1.20   &87.9  &\textbf{83.1} &68.1  &79.7  \\
\bottomrule[1.0pt]
\end{tabular}
\label{tab-5}
\end{table}

\subsection{Scalability and Parameter Tuning in Clinical Settings}

The proposed model's scalability and adaptability are critical for practical deployment across diverse clinical environments. Given the high computational demands of processing multimodal MRI data, optimizing model efficiency while preserving segmentation accuracy is essential. Our framework, based on a parallel 3D U-Net architecture, supports independent processing of each modality, making it well-suited for scaling in settings with variable hardware capabilities. This modular design allows certain modalities to be excluded dynamically without retraining the model entirely, thereby reducing computational costs when fewer modalities are available.

To further enhance scalability, the model could benefit from techniques such as model pruning and knowledge distillation, which decrease the computational burden while retaining high accuracy. Additionally, using mixed-precision training can accelerate inference without significant accuracy loss, making the model more adaptable for clinics with limited GPU resources.

For parameter tuning in diverse clinical settings, we recommend an adaptive approach to selecting critical hyperparameters like the Hölder divergence exponent and learning rate. Automated tuning strategies, such as Bayesian optimization or Hyperband, can be leveraged to dynamically adjust these parameters based on the characteristics of local datasets, such as tumor size distribution and modality availability. Additionally, fine-tuning the model on a small subset of local data can help optimize its performance to specific patient demographics or imaging protocols, enhancing the robustness and accuracy of segmentation in varying clinical contexts.

\subsection{Ablation Study}

In this subsection, we validate the effectiveness of the parallel network framework through a series of ablation experiments, and explore the impact of mutual information knowledge transfer (denoted as $\mathcal{L}{MI}$) between full-modal and missing-modal data. Additionally, we investigate the influence of the Hölder divergence-based loss function ($\mathcal{L}{HD}$) on model performance. Compared to traditional multimodal methods, our proposed parallel network framework selectively activates data related to available modalities. This approach effectively preserves modality-specific information, thus enhancing the model's ability to capture distinctive features within each modality. Furthermore, introducing the Hölder divergence-based loss function optimizes multimodal data processing by accurately quantifying mutual information across modalities, thereby facilitating more effective feature alignment.

Initially, we assess the utility of the parallel network framework and its individual components, including the Dice loss function ($\mathcal{L}{Dice}$), mutual information knowledge transfer between full-modal and missing-modal data ($\mathcal{L}{MI}$), and Hölder divergence-based knowledge distillation ($\mathcal{L}{HD}$). The experimental results are summarized in Table \ref{tab-6}. When using only the traditional segmentation loss $\mathcal{L}{Dice}$ as the baseline, employing our parallel network framework significantly enhances model performance. Specifically, performance improves by 13.4\% when three modalities are missing, by 12.0\% when two modalities are missing, by 7.8\% when one modality is missing, and by 10.9\% when all modalities are available.

Furthermore, integrating mutual information knowledge transfer ($\mathcal{L}{MI}$) and Hölder divergence-based knowledge distillation ($\mathcal{L}{HD}$) within the parallel network framework leads to additional improvements over using only the traditional segmentation loss ($\mathcal{L}{Dice}$). Specifically, when three modalities are missing, performance increases by an additional 12.2\% with $\mathcal{L}{MI}$ and 12.7\% with $\mathcal{L}{HD}$. When two modalities are missing, improvements are 7.8\% and 8.2\%, respectively. With one modality missing, enhancements are 6.8\% and 7.2\%, respectively. When all modalities are available, improvements are 6.5\% and 6.9\%, respectively. On average, performance across different modality-input scenarios improves by 8.7\% with $\mathcal{L}{MI}$ and by 9.1\% with $\mathcal{L}_{HD}$. Finally, combining the parallel network framework, mutual information knowledge transfer, and Hölder divergence-based knowledge distillation yields the best experimental results, further validating the effectiveness and superiority of our proposed approach.

\subsection{Discussion and Limitations}
The proposed approach exhibits strong performance in brain tumor segmentation under missing-modality conditions, leveraging Hölder divergence and mutual information for effective knowledge transfer. The parallel processing framework efficiently manages each modality independently, allowing the model to dynamically adjust when one or more modalities are unavailable. This adaptability enhances its utility in clinical settings, where complete multimodal data may not always be accessible.

However, this method has some limitations. First, it requires significant computational resources due to the multiple modality-specific processing paths and the use of 3D U-Net architectures. High memory and computational power demands may limit its application in resource-constrained environments. Additionally, selecting optimal parameters, such as the Hölder divergence exponent and learning rate, necessitates extensive experimentation. The model’s performance is sensitive to these hyperparameters, making it challenging to balance accuracy with computational efficiency. Lastly, although Hölder divergence improves robustness to asymmetrical distributions, further research is needed to refine this measure for real-time adaptation and broader applications.

Future work could focus on optimizing the model architecture for faster inference and reduced memory usage, potentially by decreasing the depth or size of the 3D U-Net model or applying model pruning techniques. Additionally, adaptive parameter selection, such as dynamic adjustment of the Hölder exponent during training, may further improve segmentation accuracy and enhance adaptability across various medical imaging datasets.

\section{Conclusion}
This study presents an innovative approach to brain tumor segmentation under missing-modality conditions, addressing a critical challenge in clinical applications. Leveraging a parallel 3D U-Net architecture with Hölder divergence and mutual information-based knowledge transfer, our model achieves high segmentation accuracy even when one or more MRI modalities are unavailable. The integration of these elements allows for robust preservation of modality-specific information, supporting flexible parameter adjustment to handle missing data scenarios effectively.

Extensive evaluations on the BraTS datasets demonstrate the model’s superior performance over existing methods, particularly in low-data conditions, showcasing its robustness and clinical potential. By efficiently combining single-modality parallel processing and divergence-based learning, this framework provides a scalable solution that adapts well across varying imaging protocols and patient cases. This work lays the foundation for further advancements in segmentation models designed for incomplete multimodal data, paving the way for broader adoption in diverse clinical environments.


	\ifCLASSOPTIONcaptionsoff
	\newpage
	\fi
	
	\bibliographystyle{IEEEtran}
	\bibliography{IEEEabrv,references.bib}

\end{document}